\documentclass{article}

\usepackage{iclr2027_conference,times}
\iclrfinalcopy

\usepackage[utf8]{inputenc}
\usepackage[T1]{fontenc}
\usepackage{url}
\usepackage{booktabs}
\usepackage{amsfonts}
\usepackage{nicefrac}
\usepackage{microtype}
\usepackage{xcolor}

\usepackage{amsmath,amssymb,amsthm,mathtools}
\usepackage[inline]{enumitem}
\usepackage{placeins}
\usepackage{hyperref}

\newcommand{\R}{\mathbb{R}}
\newcommand{\calX}{\mathcal{X}}
\newcommand{\calU}{\mathcal{U}}
\newcommand{\calZ}{\mathcal{Z}}
\newcommand{\calO}{\mathcal{O}}

\newcommand{\calA}{\mathcal{A}}
\newcommand{\norm}[1]{\lVert #1 \rVert}
\newcommand{\sayan}[1]{{#1}}

\providecommand{\abs}[1]{\lvert #1 \rvert}
\usepackage{placeins}

\theoremstyle{plain}
\newtheorem{theorem}{Theorem}
\newtheorem{lemma}{Lemma}
\newtheorem{proposition}{Proposition}
\newtheorem{corollary}{Corollary}

\theoremstyle{definition}
\newtheorem{definition}{Definition}
\newtheorem{assumption}{Assumption}

\theoremstyle{remark}
\newtheorem{remark}{Remark}

\theoremstyle{plain}
\newtheorem*{theorem*}{Theorem}
\newtheorem*{lemma*}{Lemma}
\newtheorem*{proposition*}{Proposition}

\title{Factored Diffusion Policies:Compositionally Generalized Robot Control
       with a Single Score Network}

\author{Sayan Mitra, Ege Yuceel, Noah Giles, Abhishek Pai \\
University of Illinois Urbana-Champaign \\
\texttt{mitras@illinois.edu}}

\begin{document}
\maketitle

\begin{abstract}
Robotic tasks are typically specified by a tuple of factors, such as the
goal to be reached, the obstacles to be avoided, and the appearance of the
scene, and collecting expert demonstrations for every combination of factor
values grows combinatorially.  We present \emph{factored diffusion policies}:
a single diffusion network trained with per-factor null-token dropout, whose
score decomposes additively across factors.  When the factors are
approximately conditionally independent given the action and observation, the
composed score approximates the joint score uniformly over all factor
combinations, so a training set in which every factor value appears at least
once suffices: the task budget grows with the sum of the factor cardinalities
rather than their product.  The decomposition residual equals the
action-gradient of the factors' interaction log-ratio and is measurable from
forward passes, so the assumption can be checked per task, and the same
structure yields a closed-loop trajectory-tube certificate.  We evaluate our approach using drone racing simulations where a drone must navigate through gates while balancing speed and path curvature against tracking overshoot which can cause gate crashes. Trained on such a cover,
the factored network passes several
times more held-out gates than the same architecture without dropout,
replicated across disjoint covers, seeds, and two and three factors, while
the same composition rule over separately trained networks generalizes far
worse.  Injected factor interaction raises the measured residual and degrades
composed inference, and our experiments indicate that the gain comes
from the factored training, and the decomposition supplies the task budget,
the certificate, and the approximate conditional independence check.
\end{abstract}

\section{Introduction}\label{sec:intro}

A standard approach to robot policy learning treats the task as
a monolithic input $z$ and trains a single policy network
$p(a \mid z)$ on a finite set of training tasks
$\calZ_{\mathrm{train}}$ by collecting expert demonstrations for
each $z \in \calZ_{\mathrm{train}}$.  Typical robotic task specifications have a factored 
structure, for example, a navigation task is specified by the relative pose of the target, obstacles to be avoided, and also the 
appearance of the environment; a landing task by allowed approach angles and environment conditions like wind; a  racing
task by the track geometry and the gate sizes.  A task
with $K$ such factors is a tuple $z = (z_1, \ldots, z_K)
\in \calZ_1 \times \cdots \times \calZ_K$, where $\calZ_i$ is
the set of values factor $i$ can take. Without exploiting the factored structure, 
the cost of data collection and training grows as $\prod_i |\calZ_i|$. In this paper, we explore how the factored 
structure of task specifications can be exploited to improve data efficiency of training and generalization to unseen factor combinations.

The cost of monolithic conditioning is unavoidable in general,
but it could shrink dramatically when the factors are
\emph{approximately conditionally independent given an
action--observation pair $(a, o)$}: knowing $(a, o)$ and the
value of one factor does not appreciably help infer the others
(Assumption~\ref{asm:approx-indep}).  For example, for a
generalizable drone racing policy, the track and the
gate-opening size can be expected to be approximately
conditionally independent given the drone state ($o$) and a
planned action chunk ($a$): the spatial route of the
planned trajectory picks out the track while its speed and
precision pick out the gate size.
Under this assumption,
each factor's effect on the policy can be learned independently
and combined post-hoc, dropping the training-task budget from
$\prod_i |\calZ_i|$ to $O\bigl(\sum_i |\calZ_i|\bigr)$ and
giving compositional generalization to unseen factor
combinations.  The remainder of the paper turns this idea
into a constructive method for diffusion policies, an approximation bound, and a
closed-loop certificate.

\paragraph{A single diffusion network with approximate conditional independence of factors.}
Diffusion policies are a leading method for learning control
policies from expert demonstrations~\citep{Chi-RSS-23,
Reuss-RSS-24}: they are trained to approximate the \emph{score}
of $p(a \mid o, z)$, the gradient $\nabla_a \log p$ in action
space.  The crucial property is that density factorization
under conditional independence becomes \emph{addition} in score
space, and \emph{approximate} conditional independence becomes
a bounded norm on the gradient of the residual factor
interaction (Theorem~\ref{thm:decomp}).  If the factors are
approximately independent given $(a, o)$, then
\[
  \nabla_a \log p(a \mid o, z_1, \ldots, z_K)
  \;\approx\; s_\varnothing + \textstyle\sum_{i=1}^{K} \Delta_i,
\]
where $s_\varnothing$ is the unconditional score and $\Delta_i$ is the score correction
for factor $i$ alone (Definition~\ref{def:corrections}).
This additive composition follows from the log-density structure;  in contrast, a behavior-cloning policy outputs a point estimate, and does not admit such a decomposition.
We exploit additive composition by
training a \emph{single shared} score network with per-factor
null-token dropout, so that each $\Delta_i$ is learned from
data covering only one factor at a time, and at inference we
compose $s_{\mathrm{comp}} = s_\varnothing + \sum_i \Delta_i$
(Definition~\ref{def:composed}); the same network also supports joint
readout, which we compare against throughout.  This identity is the
algebraic backbone of classifier-free guidance~\citep{ho2021classifierfree}
and compositional image
generation~\citep{10.1007/978-3-031-19790-1_26, pmlr-v202-du23a}, where it
composes a single concept or text-prompt factor with the
unconditional score. Our setting, has two points of distinction:
 First, the policy runs in closed loop:
the composed score field denoises samples into actions, which
a tracking controller executes on the plant, and what
ultimately matters is a task-level event such as reaching a
goal or gate passage.  Second, we quantify the gap when
conditional independence holds only approximately:
Theorem~\ref{thm:decomp} bounds
$\norm{s - s_{\mathrm{comp}}}$ by $2\sqrt{GM}$ uniformly in
$z$, where $G$ bounds the residual factor-interaction strength
(Assumption~\ref{asm:approx-indep}) and $M$ bounds the
curvature of the factor-interaction log-ratio.  Because this
bound holds for every $z \in \prod_i \calZ_i$, including
factor combinations absent from training, and $s_{\mathrm{comp}}$
needs only the unconditional score plus one per-factor
correction, training data in which every factor value appears
at least once suffices: $O(\sum_i |\calZ_i|)$ tasks against
$\Omega(\prod_i |\calZ_i|)$ for a joint model.  The residual this
theorem bounds is measurable and it is the
action-gradient of the factors' interaction log-ratio
(Proposition~\ref{prop:residual}), and therefore, the independence assumption can be
checked per task.

\paragraph{Scores also make certification possible.}
The per-factor structure also yields a closed-loop guarantee.
Fix a nominal training task $z^{\mathrm{nom}}$ and let the
\emph{nominal trajectory} be the closed-loop state trajectory
the system would follow under the true joint score on
$z^{\mathrm{nom}}$.  Under (i)~bounded factor-interaction
(Assumption~\ref{asm:approx-indep}), (ii)~per-factor Lipschitz
score (Assumption~\ref{asm:lip-score}), and (iii)~a
contracting tracking controller
(Assumption~\ref{asm:contraction}), the closed-loop trajectory
generated by the composed policy on a (possibly held-out) test
task stays within a tube of radius $R$ around the nominal
(Theorem~\ref{thm:tube}).  $R$ aggregates four sources: the
decomposition residual $2\sqrt{GM}$ from
Theorem~\ref{thm:decomp}, a score-matching error \sayan{$(2K-1)\eta$}
on training data, a per-factor extrapolation term
and a process
disturbance from sensor noise and unmodeled dynamics.  The
transfer from score-space to state-space error passes through
an ODE-sensitivity factor $C_{\mathrm{ode}}$, for which we
give \sayan{two} computable bounds following standard methods:
analytical Gr\"onwall (uniform over tasks) \sayan{and} a path-dependent
linear time-varying (LTV) bound using per-step Jacobians along
the nominal path.  If the
nominal trajectory  completes the task with a robustness margin of at least $R$, then 
 the realized trajectory under the composed policy does so as well.  In our
experiments the rigorous radius does not clear the physical gate margins, so we
use the certificate as a robustness diagnostic rather than a deployment-level
task-success guarantee.

\paragraph{Experimental validation with drone racing.}
We study the approach on drone racing: a drone (13-state underactuated
Newton--Euler quadrotor with $4$ inputs) has to fly through the gates of a
race track in prescribed order, and a task is specified by a track identity
$z_1$, a gate aperture $z_2$, and in some experiments a third factor $z_3$,
the expert speed.  The quadrotor accelerates only by tilting
its thrust axis, so the policy has to balance speed and path curvature
against tracking overshoot, and an overshoot exceeding the aperture
half-width at a crossing fails that gate.
Our experiments validate the potential for compositional generalization from
sum-sized task coverage: a single score network trained with per-factor
dropout on tasks in which every factor value appears at least once succeeds
on the majority of held-out factor combinations.  First, the shared network
passes $87\%$ of held-out gates ($26$ of $30$) under composed inference,
against $100\%$ for an oracle trained directly on the held-out tasks, while
the same composition rule over separately trained per-factor networks in the
style of~\citet{Wang-RSS-24, patil2024composing} passes $10\%$.
This leads us to conclude that parameter
sharing, not the additive score alone, helps with generalization.  Second,
a jointly conditioned network without
factor dropout matches ours; 
the performance gap between the two training modes appears at the sum-sized
coverage Theorem~\ref{thm:decomp} describes, where dropout raises held-out
passage from $3$--$16\%$ to $36$--$50\%$ with two factors and from
$5$--$11\%$ to $25$--$32\%$ with three.  Third, the gain comes from the
factored training, not from the inference rule: composed inference nearly
matches joint readout of the same network at dense coverage and trails it at sparse
coverage, as the error accounting predicts, since the composed score
combines $2K-1$ weighted network evaluations and so carries $(2K-1)\eta$ of
score-matching error where joint readout carries $\eta$.
Fourth, the residual Theorem~\ref{thm:decomp} bounds is measurable,
and when we inject a controlled factor interaction into the expert data it
rises monotonically while composed held-out passage falls from $87\%$ to
$37\%$ and joint readout is unaffected.  Additionally, we present a preliminary study of visual policies that takes
the factored conditioning to raw pixels: on single-gate traversal from an
onboard RGB camera and a noisy gyro, with four factors (venue, gate color,
gate ID, approach side), factor conditioning raises success over a baseline
without factor input and transfers zero-shot to a venue held out of training entirely.

\paragraph{Contributions.}
\begin{enumerate*}[label=(\roman*)]
  \item a factored diffusion policy realized as a single shared score network
        with per-factor null-token dropout, with a compositional
        generalization bound (Theorem~\ref{thm:decomp}) under which the
        training-task requirement grows with $\sum_i |\calZ_i|$ rather than
        $\prod_i |\calZ_i|$;
  \item an identity expressing the decomposition residual as the
        action-gradient of the factors' interaction log-ratio
        (Proposition~\ref{prop:residual}), which makes the independence
        assumption checkable per task from forward passes alone;
  \item a trajectory-tube certificate (Theorem~\ref{thm:tube}) chaining score
        error through the sampling ODE and the closed-loop tracking dynamics,
        used as a robustness diagnostic;
  \item controlled experiments that separate the mechanisms: a jointly
        conditioned no-dropout baseline, sum-sized coverage at two and three
        factors, and a controlled factor-interaction sweep
        (Section~\ref{sec:experiments}).
\end{enumerate*}

\vspace{-0.25cm}

\FloatBarrier
\section{Problem of Generalizable Control from Demonstrations}\label{sec:setup}

Consider a robot with state $x_t \in \calX \subseteq \R^n$ and
discrete-time dynamics
\begin{equation}\label{eq:dynamics}
  x_{t+1} = f(x_t, u_t),
\end{equation}
where $u_t \in \calU \subseteq \R^m$ is the control input
applied at step $t$.  
A {\em task\/} is specified by a tuple of factors
$z = (z_1, \ldots, z_K) \in \calZ_1 \times \cdots \times \calZ_K$.  Each task $z$ carries a
\emph{success predicate} $\phi_z$ over closed-loop trajectories
$x_{0:T}$ that determines whether the task is completed.
We assume an expert that generates demonstration data for a
finite training set
$\calZ_{\mathrm{train}} \subsetneq
 \calZ_1 \times \cdots \times \calZ_K$.
The expert produces a command $a_t \in \calA$ at each step, and
a low-level {\em tracking controller\/} $\kappa$ generates the final plant input
$u_t = \kappa(x_t, a_t)$.  We make the standard assumption that the closed-loop
system formed by $\kappa$ and $f$ is
contracting~\citep{lohmiller1998contraction} or incrementally
input-to-state stable~\citep{angeli2002lyapunov}, i.e., small
perturbations to $a_t$ produce bounded state deviations
(Assumption~\ref{asm:contraction}  
in Appendix~\ref{app:assumption-contraction}).
The command $a_t$ may be a single setpoint, a trajectory, or an entire plan.  
At each step the policy
also receives an observation $o_t \in \calO$, which may include
any subset of the robot state and a perceptual input (e.g., a
camera image). The factored decomposition
we develop in Section~\ref{sec:solution} is independent of the
observation modality. 
Expert demonstrations are available only for a finite training
set $\calZ_{\mathrm{train}}$
costing $O(|\calZ_{\mathrm{train}}|)$.  
For this paper, we assume the factors for the training tasks $z \in \calZ_{\mathrm{train}}$ 
are chosen independently randomly, that is, the training-task distribution factorizes as
$p_{\mathrm{gen}}(z) = \prod_{i=1}^{K} p_{\mathrm{gen}}(z_i)$.
This is a property of the data-collection protocol that holds by design.
  The goal is to design a
policy $\pi_\theta(o_t, z)$ that generalizes from
$\calZ_{\mathrm{train}}$ to the full product space
$\calZ_1 \times \cdots \times \calZ_K$, including tasks for
which no expert demonstration or training data exists.

\paragraph{Running example: drone racing.}
In  drone racing, a quadrotor has to autonomously pass through a gate using
onboard IMUs and a forward-facing RGB camera.  
The  state $x \in \R^{13}$ (position,
attitude quaternion, velocity, body angular rate) evolves under  control
$u \in \R^4$ (collective thrust, body torque) according to
Newton--Euler dynamics.  
 We study two
instantiations that share these dynamics but differ in factors
and observation: a \emph{state-based full-race} setting
($K = 2$: track $\times$ gate-opening size,
Section~\ref{sec:exp-v6}) and a \emph{vision-based single-gate}
setting ($K = 4$: venue, gate color, gate ID, approach side,
Appendix~\ref{sec:exp-sg}).
The tracking controller $\kappa$ is an $\mathrm{SE}(3)$ geometric  controller~\citep{lee2010geometric} with fixed gains.
The success predicate $\phi_z$ requires
the drone's center to pass within a gate half-width $r$ of
each gate center (in the case of full-race, in the prescribed order).
\sayan{In the state-based full-race setting, $r = r(z_2)$ varies with
the gate-opening factor $z_2$; in the
vision-based single-gate setting, $r$ is fixed.}
Every $z \in \calZ_{\mathrm{train}}$ requires careful hand-tuning of the trajectory and
a speed profile or training a specialized policy, 
or multiple expert demonstrations.

\begin{figure}[!t]
  \centering
  \includegraphics[width=\textwidth]{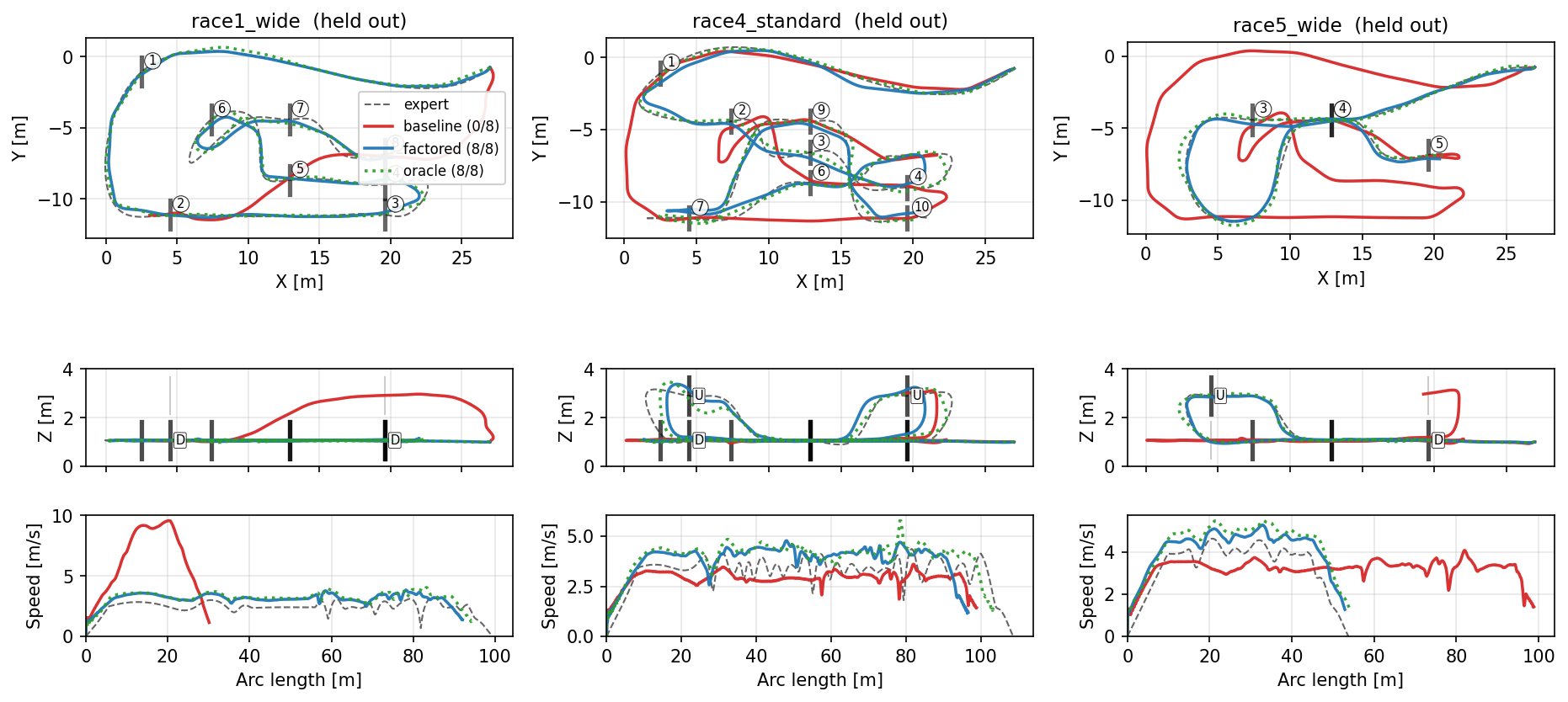}
  \caption{\small \textbf{Compositional generalization on
          held-out drone-racing tasks.}  Closed-loop trajectories
          on three (track, gate-size) pairs not seen during
          training; top-down ($X$--$Y$), side view ($X$--$Z$),
          and speed vs.\ arc length.  Black dashed: expert
          reference.  Green dotted: oracle (trained on the same
          pair).  Red: unfactored baseline.  Blue: factored
          compositional policy
          $s_{\mathrm{comp}} = s_\varnothing + \Delta_1 + \Delta_2$.
          The factored policy tracks expert and oracle to
          $\sim 0.1$\,m; the baseline diverges.  Aggregate over
          $20$ tasks: Section~\ref{sec:exp-v6}.}
  \label{fig:headline-trajectories}
\end{figure}

\section{Compositional Generalization via Conditional Independence}
\label{sec:solution}

Generalizing from $\calZ_{\mathrm{train}}$ to the full product
space $\calZ = \calZ_1 \times \cdots \times \calZ_K$ requires
additional structure: each factor's effect on the policy must
be learnable independently and combinable post hoc.
The key hypothesis for our approach to work is that the factors are
\emph{approximately conditionally independent given an
action--observation pair $(a, o)$}, that is, knowing $(a, o)$
and the value of one factor does not help infer the others
beyond what $(a, o)$ alone reveals.
For example, $z_1$ = track shape governs the direction in $a$ while
$z_2$ = gate aperture governs its speed. These are two largely
disjoint channels in $(a, o)$, so neither factor reveals the other
beyond what $(a, o)$ already does.
Formally, let $p(a, o, z)$ denote the joint distribution over
action--observation--task tuples induced by the training data:
sample a task $z \sim p_{\mathrm{gen}}$, roll out the expert
policy on task $z$, and record each action--observation pair
produced along the trajectory.  \sayan{The forward-noising
process of Section~\ref{sec:diffusion}, $a_\sigma = \alpha_\sigma a + \sigma \epsilon$
with $\epsilon \sim \mathcal{N}(0, I)$ independent of $(a, o, z)$,
extends this to a noised joint $p_\sigma(a_\sigma, o, z)$ at every
$\sigma \in [0, \sigma_{\max}]$, with $\sigma = 0$ recovering the
clean joint.}  The conditional posterior
\sayan{$p_\sigma(z \mid a_\sigma, o)$} and the per-factor marginals
\sayan{$p_\sigma(z_i \mid a_\sigma, o)$} are obtained from
\sayan{the noised joint} via Bayes'
rule.  We measure the interaction among factors through the
\emph{interaction log-ratio}
\begin{equation}\label{eq:g}
  \sayan{g_\sigma(z, a_\sigma, o)
    \;:=\; \log \frac{p_\sigma(z \mid a_\sigma, o)}
                     {\prod_{i=1}^{K} p_\sigma(z_i \mid a_\sigma, o)},
  \qquad z = (z_1, \ldots, z_K),\;\;\sigma \in [0, \sigma_{\max}].}
\end{equation}
Conditional independence \sayan{at noise level $\sigma$} holds at
\sayan{$(a_\sigma,o)$} if and only if
\sayan{$g_\sigma(z, a_\sigma, o) = 0$} for all $z$.

\begin{assumption}[Bounded factor interaction]
  \label{asm:approx-indep}
  \sayan{There exists $G \geq 0$ such that for every noise level
  $\sigma \in [0, \sigma_{\max}]$, every action--observation pair
  $(a_\sigma, o)$ in the support of $p_\sigma$, and every
  $z \in \calZ_1 \times \cdots \times \calZ_K$,
  $|g_\sigma(z, a_\sigma, o)| \leq G$.}
\end{assumption}

\vspace{-0.25cm}

Even when the training prior $p_{\mathrm{gen}}(z) = \prod_i p_{\mathrm{gen}}(z_i)$
factorizes, conditioning on the action $a$ can induce dependence between
factors (the collider effect), and conditional diffusion models can in general
violate strict conditional independence~\citep{gaudi2025coind}.
Assumption~\ref{asm:approx-indep} quantifies the strength of such effects via
$G$, and Proposition~\ref{prop:residual} makes the resulting score residual
measurable.  On our state-based benchmark the norm of the measured score
residual averages $0.45$ on training tasks
and $1.08$ on held-out tasks, in the network's output units: small where the
assumption is interpolated, and roughly doubled where it is extrapolated
(Section~\ref{sec:experiments}).

\subsection{Diffusion Policies, Denoising ODEs, and ODE Sensitivity}
\label{sec:diffusion}

Let $p(a \mid o, z)$ denote the conditional distribution over
commands $a \in \calA$ given observation $o$ and task $z$,
induced by the expert demonstrations.
We learn to approximate $p(a \mid o, z)$ with a diffusion model, which operates on
its \emph{score function}, the gradient of the log-density with
respect to the action.  The score is the operational object on which
the compositional structure can be implemented.

A diffusion model defines a forward noising process
$a_\sigma = \alpha_\sigma a_0 + \sigma \epsilon$ with
$\epsilon \sim \mathcal{N}(0,I)$ and
$\alpha_\sigma^2 + \sigma^2 = 1$ ($\sigma = 0$
is clean data, large $\sigma$ is near-Gaussian), where
$a_0 \sim p(a \mid o, z)$.  The score at noise level $\sigma$ is
\begin{equation}\label{eq:score}
  s(a_\sigma, \sigma, o, z)
    \;:=\; \nabla_{a_\sigma} \log p_\sigma(a_\sigma \mid o, z),
\end{equation}
where $p_\sigma$ is the marginal density of $a_\sigma$.  A
neural network $s_\theta$ approximates $s$ via denoising
score matching.  Actions are generated by solving the
reverse sampling ODE
\begin{equation}\label{eq:reverse-ode}
  \frac{da_\sigma}{d\sigma}
    = f_{\mathrm{dr}}(\sigma)\, a_\sigma
      + \tilde{g}(\sigma)\, s_\theta(a_\sigma, \sigma, o, z)
\end{equation}
backward from $\sigma_{\max}$ (near-Gaussian) to
$\sigma \approx 0$ (near-clean), where $f_{\mathrm{dr}}$ and
$\tilde{g}$ are schedule-dependent drift and diffusion
coefficients.  In practice the integration is performed with a
discrete-time scheduler such as DDIM (Denoising Diffusion
Implicit Models~\citealp{song2021denoising}), which under a
deterministic-seed setting realizes~(\ref{eq:reverse-ode}).
At inference on an unseen task the policy runs a \emph{composed
learned} score that differs from the true joint score, and we need to bound
how much this difference affects the generated command.  The bound below is a
standard Gr\"onwall argument on~(\ref{eq:reverse-ode}); similar arguments
appear in convergence analyses for diffusion
samplers~\citep{song2021scorebased, chen2023iclr-sampling,ICLR2024_9edce071}.
Throughout the sensitivity analysis both denoising chains start from a shared
initial sample $a_{\sigma_{\max}} = a_{\sigma_{\max}}' = a^\star$; since
\eqref{eq:reverse-ode} is a probability-flow ODE this makes the sampled
action a deterministic function of $(o, z)$ and yields a per-task certificate.
Section~\ref{sec:exp-v6} verifies that the conclusions transfer to
fresh-seed sampling at deployment.

\begin{lemma}[ODE Sensitivity]\label{lem:ode-sens}
  Let $s, s' : \calA \times [0, \sigma_{\max}] \to \calA$ be two
  score fields with
  (a) $\norm{s - s'}_\infty \leq \varepsilon$ and
  (b) $s$ is $L_a(\sigma)$-Lipschitz in its first argument.
  Let $\{a_\sigma\}, \{a_\sigma'\}$ be denoising paths generated
  by the reverse ODE~\eqref{eq:reverse-ode} from a shared initial
  noise $a_{\sigma_{\max}} = a_{\sigma_{\max}}'$.  Then
  \begin{equation}\label{eq:ode-bound}
    \norm{a_0 - a_0'}
      \;\leq\; C_{\mathrm{ode}} \cdot \varepsilon,
    \ \
    C_{\mathrm{ode}}
      \coloneqq \int_0^{\sigma_{\max}}
          |\tilde{g}(\sigma)|\,
          \exp\!\Bigl(
            \int_0^{\sigma}
              \bigl[|f_{\mathrm{dr}}(\sigma')|
                    + |\tilde{g}(\sigma')|\, L_a(\sigma')
              \bigr] d\sigma'
          \Bigr)\, d\sigma.
  \end{equation}
\end{lemma}

Lemma~\ref{lem:ode-sens} is conservative: the scalar Lipschitz bound
$L_a(\sigma)$ collapses the directional structure of the score Jacobian, and
the Gr\"onwall integral compounds the worst case at every noise level.
Linearizing the discrete DDIM map around the nominal denoising trajectory
$\{\bar a_k\}$ yields a state-transition matrix whose operator norm replaces
the scalar Lipschitz product, capturing the anisotropy of the per-step
Jacobians and inter-step cancellation~\cite{fan2016locally}.  The resulting
\emph{path-dependent LTV} constant $C_{\mathrm{ode}}^{\mathrm{LTV}}(\bar a)$
(Lemma~\ref{lem:ltv-sens}) is about five orders of magnitude tighter than the
Gr\"onwall constant in our experiments (Section~\ref{sec:exp-v6}).

\subsection{Factored Score Decomposition}\label{sec:factored}

\begin{definition}[Composed score]
  \label{def:corrections}\label{def:composed}
  Let $z = (z_1, \ldots, z_K)$ be a task.  For each
  $i \in \{1, \ldots, K\}$, the \emph{score correction} for
  factor $i$ is
   $ \Delta_i(a_\sigma, \sigma, o, z_i)
      = s(a_\sigma, \sigma, o, z_i) - s(a_\sigma, \sigma, o)$,
  where $s(\cdot, z_i)$ is the score conditioned on factor $i$
  alone (other factors marginalized via null-token substitution)
  and $s(\cdot)$ is fully unconditional.  The \emph{composed
  score} for $z$ is
  \begin{equation}\label{eq:composed}
    s_{\mathrm{comp}}(a_\sigma, \sigma, o, z)
      = s(a_\sigma, \sigma, o)
        + \sum_{i=1}^{K} \Delta_i(a_\sigma, \sigma, o, z_i).
  \end{equation}
\end{definition}

\begin{theorem}[Decomposition error bound]\label{thm:decomp}
  Suppose Assumption~\ref{asm:approx-indep} holds and \sayan{for every
  $\sigma \in [0, \sigma_{\max}]$,} the interaction log-ratio
  \sayan{$g_\sigma$} defined in~\eqref{eq:g} is twice differentiable in
  \sayan{$a_\sigma$} with
  \sayan{$\norm{\nabla_{a_\sigma}^2\, g_\sigma(z, a_\sigma, o)}_{\mathrm{op}} \leq M$}
  uniformly \sayan{in $\sigma, a_\sigma, o, z$}.  Then for all $a_\sigma,\, \sigma,\, o,\, z$,
  \begin{equation}\label{eq:decomp-bound}
    \norm{s(a_\sigma, \sigma, o, z)
          - s_{\mathrm{comp}}(a_\sigma, \sigma, o, z)}
    \leq 2\sqrt{G M}.
  \end{equation}
\end{theorem}

The composed score $s_{\mathrm{comp}}$ approximates the true joint
score $s(\cdot, z)$ within $2\sqrt{G M}$ for \emph{every}
$z \in \calZ_1 \times \cdots \times \calZ_K$, including tuples absent from
$\calZ_{\mathrm{train}}$.  Computing $s_{\mathrm{comp}}$
requires $K + 1$ learned components
(Definition~\ref{def:corrections}): the unconditional score
$s(\cdot)$ and one per-factor score $s(\cdot, z_i)$ for each
factor $i$.  Learning $s(\cdot, z_i)$ requires training data
in which the value $z_i$ appears with at least one setting of
the other factors.  A training set of size
$|\calZ_{\mathrm{train}}| = O\bigl(\sum_i |\calZ_i|\bigr)$ in
which every factor value appears in at least one training task
therefore suffices to evaluate $s_{\mathrm{comp}}$ on all
$\prod_i |\calZ_i|$ tasks.  A joint (non-compositional) model
that learns $s(\cdot, z)$ directly requires training tasks
covering all $\Omega\bigl(\prod_i |\calZ_i|\bigr)$ factor
combinations.

We emphasize that this is a factor-value
\emph{coverage} statement: it counts which training tasks must
appear, not how many demonstrations per task are needed to
estimate each score to a target accuracy; estimation quality
enters separately through the measured score-matching error
$\eta$ in~\eqref{eq:kappa}.
The quantity Theorem~\ref{thm:decomp} bounds has an exact form, which makes the assumption checkable.

\begin{proposition}
  \label{prop:residual}
  If the scores in Definition~\ref{def:corrections} are exact, that is
  $s(a_\sigma, \sigma, o, z_i) = \nabla_{a_\sigma} \log p_\sigma(a_\sigma \mid o, z_i)$
  and $s(a_\sigma, \sigma, o) = \nabla_{a_\sigma} \log p_\sigma(a_\sigma \mid o)$,
  then for all $a_\sigma, \sigma, o, z$,
  \begin{equation}\label{eq:residual}
    s(a_\sigma, \sigma, o, z) - s_{\mathrm{comp}}(a_\sigma, \sigma, o, z)
      = \nabla_{a_\sigma}\, g_\sigma(z, a_\sigma, o).
  \end{equation}
\end{proposition}

The residual therefore vanishes identically under conditional independence
and is controlled by Theorem~\ref{thm:decomp} otherwise.  Evaluated on a
trained network rather than on exact scores, its norm, which we denote
$\varepsilon_D$, estimates $\norm{\nabla_{a_\sigma} g_\sigma}$ together with the
network's own score error, which cannot be separated from it; $\varepsilon_D$
is an estimator of the assumption's violation rather than an exact reading of
it.  It needs only forward passes and no ground-truth score, and it can be
computed for a factor tuple before deploying on it.

The decomposition $s_{\mathrm{comp}} = s + \sum_i \Delta_i$ is
not fully determined by training data: a constant can shift between
$s$ and any $\Delta_i$ without changing $s_{\mathrm{comp}}$.  We
resolve this via \emph{factor dropout} during training, replacing
each factor conditioning independently with a learned null token
with probability $p_{\mathrm{drop}}$, so $s$ is identified as the
all-null behavior and each $\Delta_i$ as the residual when only
factor $i$ is revealed (Appendix~\ref{app:identifiability}).  
 A network trained with
$p_{\mathrm{drop}} = 0$ also receives every factor, but only ever jointly.

\subsection{Trajectory Tube Certificate}\label{sec:certificate}

In our setting the expert controller is deterministic: for each
$(o, z)$ there is a unique reference chunk
$a_0^*(o, z) \in \calA$, so
$p(a_0 \mid o, z) = \delta(a_0 - a_0^*(o,z))$, where $\delta$
denotes the Dirac delta distribution (a point mass at $a_0^*$).
The noised density is then a Gaussian centered at
$\alpha_\sigma a_0^*$, and the true score has the closed form
  $s(a_\sigma, \sigma, o, z) = -\frac{a_\sigma - \alpha_\sigma\, a_0^*(o, z)}{\sigma^2}$.
This closed form makes
the score-matching error $\eta$~\eqref{eq:kappa} directly
measurable from training data.

Theorem~\ref{thm:tube} additionally needs a per-factor Lipschitz
regularity of the trained score with respect to each factor
$z_i$ (Assumption~\ref{asm:lip-score}); this is a
mild architectural property that supplies the per-factor
sensitivity constants $L_i$.
The remaining ingredient is the network's approximation error.
The score-matching error $\eta$ is a \emph{measured} quantity: 
the worst-case gap between $s_\theta$ and the
true score on the training set,
\begin{equation}\label{eq:kappa}
  \eta
    \;:=\;
    \sup_{\substack{a_\sigma,\, \sigma,\, o \\
                    c \,\in\, \calZ_{\mathrm{train}} \cup \{\varnothing\}}}
    \norm{s_\theta(a_\sigma, \sigma, o, c)
          - s(a_\sigma, \sigma, o, c)},
\end{equation}
where $c$ ranges over all conditioning configurations used in
the composed score (unconditional, each single-factor, and the
joint); how $\eta$ is estimated is discussed in
Appendix~\ref{app:cert-extra}.  Expanding the composed score
(Definition~\ref{def:composed}) gives
$s^{\mathrm{comp}} = \sum_{i=1}^K s(\cdot, z_i) - (K-1)\, s_\varnothing$,
so the triangle inequality gives
\begin{equation}\label{eq:composed-kappa}
  \norm{s_\theta^{\mathrm{comp}} - s^{\mathrm{comp}}}
    \leq \sayan{(2K-1)\eta}.
\end{equation}

\begin{theorem}[Factored trajectory tube]\label{thm:tube}
  Under
  Assumptions~\ref{asm:approx-indep},  \ref{asm:contraction}, and
  \ref{asm:lip-score}, 
  and either
  (a)~the $L_a(\sigma)$-Lipschitz condition  of
  Lemma~\ref{lem:ode-sens}, or
  (b)~the Jacobian regularity of Lemma~\ref{lem:ltv-sens}, 
  let $x_t^{\mathrm{nom}}$ be the closed-loop state trajectory
  under the \emph{true} joint score $s(\cdot, z^{\mathrm{nom}})$
  for a nominal task $z^{\mathrm{nom}} \in \calZ_{\mathrm{train}}$,
  and let $x_t$ be the trajectory under the
  \emph{composed learned} score
  $s_\theta^{\mathrm{comp}}(\cdot, z)$ with per-factor
  perturbations
  $\norm{z_i - z_i^{\mathrm{nom}}} \leq \delta_i$ for each
  $i \in \{1, \ldots, K\}$, starting from
  $\norm{x_0 - x_0^{\mathrm{nom}}} \leq \delta_0$.
  Let $\eta$ be the score-matching error~\eqref{eq:kappa},
  measured on $\calZ_{\mathrm{train}}$.  Then
  \begin{equation}\label{eq:tube}
    \norm{x_t - x_t^{\mathrm{nom}}}
      \leq \lambda^t\, \delta_0
           + \frac{B_\kappa\, C_{\mathrm{ode}}\,
                   \varepsilon_s + w}
                  {1 - \lambda},
  \end{equation}
  where the total score perturbation decomposes as
  $\varepsilon_s \leq 2\sqrt{GM} + \sayan{(2K-1)\eta}
   + \sum_{i=1}^{K} L_i\, \delta_i$,
  with terms from Theorem~\ref{thm:decomp}, the network approximation
  error~\eqref{eq:composed-kappa}, and the factor perturbation.
\end{theorem}

\noindent
$C_{\mathrm{ode}}$ is instantiated either by the Gr\"onwall bound
$C_{\mathrm{ode}}^{\mathrm{ana}}$ (Lemma~\ref{lem:ode-sens}, uniform over
$z$) or by the path-dependent LTV bound
$C_{\mathrm{ode}}^{\mathrm{LTV}}(\bar a^{(z)})$ (Lemma~\ref{lem:ltv-sens},
per task); under hypothesis~(b) the bound is first-order in $\varepsilon_s$
(Appendix~\ref{app:cert-extra}).  After the transient decays the error is
bounded by $R_{\mathrm{ss}} = (B_\kappa\, C_{\mathrm{ode}} \varepsilon_s +
w)/(1 - \lambda)$ (Corollary~\ref{cor:ss}).  For drone racing this
specializes to a per-gate criterion: with $R = \delta_0 + R_{\mathrm{ss}}$, if
the nominal trajectory passes through every gate center and $R < r(z_2)$, then
the trajectory generated by the composed learned score passes every gate.
In our experiments, the rigorous $R_{\mathrm{ss}}$ does not clear the gate
half-widths and only an empirical variant does
(Appendix~\ref{app:full-race}), so it serves as a robustness
diagnostic rather than a gate passage certificate; the factored bound $\sum_i L_i \delta_i$ is compared with the
joint bound in Appendix~\ref{app:cert-extra}.

\section{Experiments: Agile Aerial Robotics}
\label{sec:experiments}

Drone racing is a demanding benchmark for generalizable control:
track geometry, gate-opening size, reference speed, visual venue, and gate color
 can all be varied independently.  In practice each
track is a separate engineering problem; competition teams spend
days to weeks per configuration~\citep{kaufmann2023champion,song2023reaching}, and
joint generalization across these axes is not standard.
Section~\ref{sec:exp-v6} holds out unseen track $\times$ gate-size
pairs on state-based full races; a preliminary vision-based single-gate study
is in Appendix~\ref{sec:exp-sg}.

All our runs use a 0.5\,kg, 25\,cm racing-class quadrotor, simulated as a 6-DOF Newton--Euler 
rigid body at $600$\,Hz.  
Three diffusion policies differing only in factor conditioning are discussed in the following:
{\em Baseline} (no factor input);
{\em Factored} (ours: per-factor embeddings with
$p_{\mathrm{drop}}{=}0.1$ null-token dropout; composed score
$s_{\mathrm{comp}}{=}s_\varnothing+\sum_i\Delta_i$,
Definition~\ref{def:composed});
{\em Oracle} (factored, trained on all tasks including held-out;
architectural upper bound).  Baseline and Factored never see the
held-out tasks.

\subsection{State-Based Full-Race Composition}
\label{sec:exp-v6}

The quadrotor must complete full multi-gate races (2--10 gates per
track) using  state-based control.  The
diffusion model generates a trajectory of $32$
keypoints $a^j = (p_j, v_{\mathrm{speed},j}) \in \R^4$ per task;
at inference, keypoints are
linearly interpolated to $256$ dense waypoints (velocity from
unit tangent $\times$ speed; yaw from $\operatorname{atan2}(v_y,
v_x)$) and tracked by the low-level $\mathrm{SE}(3)$
controller.
Eight race tracks (race1--8) share a common gate layout but different gate sequences.
Two factors define a task: track identity
$z_1 \in \{\text{race1},\ldots,\text{race8}\}$ and gate aperture
(opening size) $z_2 \in \{\text{narrow}\ (r{=}0.3\,\mathrm{m}),
\text{standard}\ (0.762\,\mathrm{m}),
\text{wide}\ (1.0\,\mathrm{m})\}$.
The two factors are approximately conditionally independent: $z_1$ sets the
geometric raceline, while $z_2$ primarily modulates the speed profile
along that path (narrower gates demand slower passage). 
Of $8\times 3 = 24$ tasks, $20$ are feasible; per-task max
feasible speed is found by binary search.  The denoiser is a 1D temporal ConvNet (4.3M parameters)
with $\mathrm{Emb}(8,64)$ for track and $\mathrm{Emb}(3,64)$ for
gate-size, each with a learned null token, trained from $50$
DAgger-style rollouts per task.  Six (track, gate-size) pairs are
held out from training: (race1, wide), (race4, standard),
(race5, wide), (race6, narrow), (race7, standard),
(race8, wide).  Every factor level appears in at least one
training task: each track in $1$--$2$ of its feasible gate
sizes, and narrow / standard / wide in $3$ / $6$ / $5$ training
tasks respectively. Only the joint pairs are unseen.
Coverage is thinnest for narrow, which is feasible on four tracks
and held out on one of them; §\ref{app:fr:per-task} shows that
this is also where passage is weakest. Track descriptions, the
 speed maps, per-task results,  ablations, and
 certificate results appear in
Appendix~\ref{app:full-race}.

\paragraph{Main results.}
Table~\ref{tab:exp-aggregate} reports closed-loop gate passage across all
$20$ feasible tasks ($104$ gates), split into training ($74$ gates over $14$
tasks) and held-out ($30$ gates over $6$ tasks).  
Every row comes from one
evaluation run on the listed checkpoints, and each gate is scored
independently, so a missed gate does not affect the scoring of later ones.
The factored model with composed inference passes $98/104$ ($94\%$) overall
and $26/30$ ($87\%$) held out, against $103/104$ and $30/30$ for an oracle
trained directly on the held-out tasks
(Figure~\ref{fig:headline-trajectories}).  The four missed held-out gates fall
on three combinations, two of them on (race6, narrow), which the oracle passes
$3/3$: this is a generalization failure rather than a control limit.  Narrow is
the thinnest-covered aperture, feasible on four tracks and held out on one of
them.

\begin{table}[h]
\centering
\small
\caption{\small Gate passage on full multi-gate races: trained on $14$ of
         the $20$ feasible tasks, $6$ held out.  One evaluation run; each gate scored
         independently; no row crashes.  The $K$-network row applies the same
         composition rule as ours with $12$ separately trained networks.
         The no-dropout rows are our architecture trained with
         $p_{\mathrm{drop}} = 0$, receiving the full task tuple on every
         sample.}
\label{tab:exp-aggregate}
\begin{tabular}{lccc}
  \toprule
  Method & All (104) & Training (74) & Held-out (30) \\
  \midrule
  Baseline (no factor input)          & 19/104 (18\%)  & 16/74 (22\%) & 3/30 (10\%)  \\
  $K$-network composition (12)        & 61/104 (59\%)  & 58/74 (78\%) & 3/30 (10\%)  \\
  No dropout, joint readout           & 101/104 (97\%) & 72/74 (97\%) & 29/30 (97\%) \\
  No dropout, composed                & 94/104 (90\%)  & 66/74 (89\%) & 28/30 (93\%) \\
  \textbf{Factored, composed (ours)}  & \textbf{98/104 (94\%)} & \textbf{72/74 (97\%)} & \textbf{26/30 (87\%)} \\
  \textbf{Factored, joint readout (ours)} & \textbf{101/104 (97\%)} & \textbf{72/74 (97\%)} & \textbf{29/30 (97\%)} \\
  Oracle (all $20$ tasks)             & 103/104 (99\%) & 73/74 (99\%) & 30/30 (100\%) \\
  \bottomrule
\end{tabular}
\end{table}

\paragraph{Shared parameters plus factor dropout are the mechanism.}
Table~\ref{tab:exp-aggregate} separates the candidate explanations.  The
$K$-network row applies the identical composition rule over twelve separately
trained networks and passes $10\%$ of held-out gates where the shared network
passes $87\%$, so the additive formula alone does not explain the result; since
it changes parameter sharing, data sharing and optimization difficulty at once,
we cannot draw conclusions from it.  The controlled comparison is the
no-dropout network, which differs from ours only in $p_{\mathrm{drop}}$: on the
$14$-task split it matches our joint readout exactly, $101/104$ and $29/30$.
At this coverage the gains therefore come from giving one shared network
factor information, not from dropout and not from test-time composition, and
the $14$-task split is an interpolation benchmark rather than a test of
Theorem~\ref{thm:decomp}.  Thinning coverage separates them
(Table~\ref{tab:exp-coverage}).  At a sum-sized cover of $8$ tasks in which
every factor value appears, dropout raises held-out passage from $3$--$16\%$
to $36$--$50\%$ of the $58$ uncovered gates, on two disjoint covers and at
two training seeds; with a third factor $z_3$, the expert speed as a fraction
($0.70$, $0.85$, $1.00$) of each task's feasible maximum, $8$ of the $60$
feasible tuples form a sum-sized cover and dropout raises passage on the
$266$ uncovered gates from $5$--$11\%$ to $25$--$32\%$.
Composed inference never beats joint readout of the same network: it trails
by three gates on the $14$-task split ($98$ against $101$ of $104$) and
further at the sum-sized covers ($13$--$14$ against $21$--$29$ of $58$; at
$K = 3$, $46$ and $48$ against $74$ and $72$ of $266$), consistent with its
$(2K-1)\eta$ score-matching term.  What composition provides is the analysis:
the sum-rather-than-product task budget is proved through the composed score,
Theorem~\ref{thm:tube}'s radius factors into ODE sensitivity times a per-factor
error budget only because the score is a sum of per-factor terms, and the
residual is measurable per task (Proposition~\ref{prop:residual}).  Full
$K$-network details are in Appendix~\ref{app:fr:knet}.

\begin{table}[h]
\centering
\small
\caption{\small Held-out gate passage as training coverage thins, joint
         readout.  Each held-out set is every feasible combination outside
         the cover, so compare within a row.  Covers A and B are disjoint.}
\label{tab:exp-coverage}
\begin{tabular}{llccc}
  \toprule
  Factors & Training tasks & Cover, seed & $p_{\mathrm{drop}} = 0.1$ & $p_{\mathrm{drop}} = 0$ \\
  \midrule
  $K=2$ & $14$ of $20$ (Table~\ref{tab:exp-aggregate}) & --, 42 & 29/30 & 29/30 \\
        & $11$ of $20$             & --, 42 & 27/42 & 21/42 \\
        & $8$ of $20$ (sum-sized)  & A, 42  & 22/58 & 4/58  \\
        &                          & A, 43  & 21/58 & 6/58  \\
        &                          & B, 42  & 24/58 & 2/58  \\
        &                          & B, 43  & 29/58 & 9/58  \\
  \midrule
  $K=3$ & $8$ of $60$ (sum-sized)  & A, 42  & 74/266 & 23/266 \\
        &                          & A, 43  & 85/266 & 29/266 \\
        &                          & B, 42  & 72/266 & 14/266 \\
        &                          & B, 43  & 66/266 & 17/266 \\
  \bottomrule
\end{tabular}
\end{table}

\paragraph{Composition under forced factor interaction.}
To test Assumption~\ref{asm:approx-indep} where it fails by construction, we add a
non-additive term $\alpha \cdot \mathrm{narrowness}(z_2) \cdot
\mathrm{curvature}(z_1)$ to the expert speed profile, leaving geometry and gate
positions unchanged so the success criterion is identical, and retrain at each
$\alpha$ on the $14$-task split.  As $\alpha$ grows the measured residual
$\varepsilon_D$ rises monotonically and composed held-out passage falls, gently
to $\alpha = 0.75$ and then sharply, while joint readout is unaffected
(Table~\ref{tab:exp-interaction}).  One seed per level brackets the transition
between $\alpha = 0.75$ and $1.0$ without fixing a threshold.

\begin{table}[h]
\centering
\small
\caption{\small Injected factor interaction, $14$-task split, one seed per
         level.  $\varepsilon_D$ is the training-task mean residual.}
\label{tab:exp-interaction}
\begin{tabular}{lcccccc}
  \toprule
  $\alpha$ & 0 & 0.25 & 0.5 & 0.6 & 0.75 & 1.0 \\
  \midrule
  $\varepsilon_D$              & 0.45 & 0.47 & 0.52 & 0.55 & 0.58 & 0.64 \\
  Composed, held-out (of 30)   & 26 & 26 & 24 & 23 & 22 & 11 \\
  Joint readout, all (of 104)  & 101 & 101 & 101 & 101 & 101 & 101 \\
  \bottomrule
\end{tabular}
\end{table}

\paragraph{Certificates.}
We instantiate Theorem~\ref{thm:tube} with the path-dependent
LTV sensitivity of Lemma~\ref{lem:ltv-sens}, obtaining finite
task-wise sensitivity bounds that are six orders of magnitude
tighter than the uniform Gr\"onwall bound.  The resulting
norm-based tubes remain conservative relative to the physical
gate margins, so we treat them as a robustness diagnostic
rather than a deployment-level certificate.  
Further details on  empirical  measurements and residual-gap
analysis are  in Appendix~\ref{app:full-race}.

\paragraph{Vision-based single-gate traversal (Appendix~\ref{sec:exp-sg}).}
A jointly factor-conditioned diffusion policy, trained without factor dropout
and driving from a single onboard camera and a noisy gyro, reproduces the
state-based racing picture on a four-factor task space.  Within a venue, where
every factor value is seen and only their combinations are held out (the
setting of Theorem~\ref{thm:decomp}), factor input adds little: $59\%$ against
$58\%$ held-out success for a baseline without it, the same interpolation
regime as the state-based experiments reported above.  Under distribution shift it adds
a lot: $63\%$ against $42\%$ on the training venues and $37\%$ against $25\%$
on a venue absent from training, the latter a factor \emph{value} the policy
never saw and so outside the theorem's coverage.  As in the state-based task,
the crash rate is governed by training length rather than by factor dropout,
which at matched early stopping leaves crashes unchanged and costs success.

\vspace{-0.25cm}
\section{Related Work}\label{sec:related}

Diffusion Policy~\citep{Chi-RSS-23} and its multi-task,
3D, and flow-matching extensions~\citep{Reuss-RSS-24,
Ze-RSS-24, BlackK-RSS-25} treat task variation as a monolithic
conditioning input.  Compositional diffusion in image
generation~\citep{10.1007/978-3-031-19790-1_26, pmlr-v202-du23a, ho2021classifierfree,
NEURIPS2023_6f125214} established the additive score identity we
adapt to control.
Several recent works also compose diffusion policies for
robotics, each on a different compositional axis: heterogeneous
data sources (PoCo,~\citealp{Wang-RSS-24}), test-time policy
mixtures (GPC,~\citealp{cao2026compose}), latent behavioral modes
(FDP,~\citealp{11395999}), input modalities
(Modality-Composable DP,~\citealp{cao2025modality}), or
base motion skills for few-shot adaptation (DSE,~\citealp{patil2024composing}).  All combine
$K$ separately trained networks at inference, whereas we use a single
shared score network with per-factor null-token dropout,
named factors with held-out combination generalization, and a
closed-loop trajectory-tube certificate.  

Certified-robustness frameworks~\citep{NEURIPS2019_95e1533e,
everett2021reachlp, mitra2024perception}
treat the policy monolithically;
sequential skill composition~\citep{Liang_2024_CVPR,
pmlr-v229-mishra23a} operates across time rather than across
factors;
champion-level drone racing~\citep{kaufmann2023champion,
song2023reaching, foehn2021time} tunes per track;
domain randomization~\citep{tobin2017domain, Sadeghi-RSS-17}
and task-conditioned multi-task RL randomize or condition over
factors at training without explicit per-factor composition.
Extended discussions are in
Appendix~\ref{app:related-work}.

\vspace{-0.25cm}
\section{Limitations}\label{sec:limitations}
\vspace{-0.25cm}

All experiments are in simulation and within a single robotic
domain (drone racing); transfer of the method to manipulation
or to real hardware is not yet established.  The analysis
rests on Assumption~\ref{asm:approx-indep}: the benchmark
factors were chosen to be approximately conditionally independent.  The
measured residual roughly doubles on held-out tasks, and under an injected
interaction composed inference degrades while joint readout does not
(Table~\ref{tab:exp-interaction}); a domain whose factors interact strongly by
nature is not tested.  The trajectory-tube certificate should be read as a
robustness diagnostic rather than a closed task-success
guarantee: the rigorous steady-state radius clears no gate
half-width (Appendix~\ref{app:fr:eps-s}), and only the
empirical variant does.  The certificate moreover
closes only for the state-based
full-race policy and does not transfer to the vision-based
single-gate policy:  a gap that is structural to operator-norm path
analysis on image-conditioned scores.  Closing it likely requires
a non-norm-based, manifold-aware sensitivity analysis which we
leave to future work.

\subsection*{AI use statement}
In this work we used generative AI tools (LLM-based coding assistants) for
the following tasks: designing and providing feedback
on experiments, implementing experiments, interpreting the experimental results,
drafting and editing parts of the paper, creating figures, identifying and
summarizing related literature.
We did not use generative AI tools to develop the theoretical framework or to
formulate or prove Theorems. Translation,
dataset cleaning and qualitative analysis are not applicable to this work.

We have reviewed all AI-assisted work.  Every reference was checked
individually against its published record 
before inclusion.  Every number reported in the paper is produced by
a released script from a released checkpoint under a fixed seed, and the
central results were independently re-run by a second author on a separate
machine.  AI-assisted parts of the code were reviewed and tested by the
authors.  We take responsibility for the final content of this work, including
text, claims or artifacts produced with the aid of generative AI.

\subsection*{Reproducibility statement}
The supplementary material contains the code for data generation, training,
evaluation, and the certificate and decomposition-error measurements, together
with the trained checkpoints and the scene assets for the vision benchmark.
The expert demonstrations are not included because of their size, but they can
be regenerated with the included data-collection scripts. The supplementary
documents the commands and settings used for the main experiments; further
experimental details are given in Section~\ref{sec:experiments} and the
appendix. Data generation, training, and evaluation are seeded, and results we
report across multiple seeds are stated as such. Because the policy is
evaluated in closed loop, small numerical differences in network inference
across GPUs and software versions can change individual rollouts; the central
results were independently re-run by a second author on a separate machine
and reproduced the reported findings.

\subsection*{Ethics statement}
All experiments are in simulation.  Autonomous-flight policies that generalize
from fewer training tasks have potential positive applications in inspection,
search-and-rescue and delivery, and dual-use risks in surveillance and
weaponised aircraft that are common to autonomous-aviation research.
Standard deployment-stage mitigations such as geofencing, operator
authentication and regulatory compliance apply and are outside the scope of
this simulation-only work.  No human subjects or personal data are involved.

{\small
  \bibliographystyle{iclr2027_conference}
  \bibliography{refs}
}

\appendix

\section{Formal Assumptions and Proofs}\label{app:proofs}

This appendix collects the full formal statement of the
cascade-contraction assumption and the proofs of all results
stated in the main paper.  Each result is restated for
self-containedness; numbering refers to the main-paper labels.

\subsection{Cascade contraction assumption}
\label{app:assumption-contraction}

\begin{assumption}[Cascade contraction]\label{asm:contraction}
  The closed-loop system consisting of the tracking controller
  $\kappa$ and the plant dynamics $f$ is contracting: for any
  two command sequences $\{a_t\}$ and $\{a_t'\}$ and corresponding
  state trajectories $\{x_t\}$ and $\{x_t'\}$ satisfying
  $x_{t+1} = f(x_t, \kappa(x_t, a_t))$,
  \begin{equation}\label{eq:contraction}
    \norm{x_{t+1} - x_{t+1}'}
      \leq \lambda\, \norm{x_t - x_t'}
           + B_\kappa\, \norm{a_t - a_t'}
           + w,
  \end{equation}
  where $\lambda \in [0,1)$ is the per-step contraction rate,
  $B_\kappa > 0$ is the command-to-state gain, and $w \geq 0$ is
  the per-step disturbance from sensor noise and unmodeled
  dynamics.
\end{assumption}

This assumption is a property of the tracking controller and
plant, independent of how the commands $a_t$ are generated
(by a diffusion policy, a hand-tuned planner, or any other
source).  Contraction ensures that state deviations caused by
command perturbations remain bounded rather than accumulating
over time.  On its own, contraction does not yield a useful
certificate: it bounds state deviation in terms of
\emph{command} perturbation $\norm{a_t - a_t'}$, but says
nothing about how large that perturbation is or how it depends
on the task.  The factored diffusion policy fills this gap by
decomposing the command perturbation into per-factor
contributions $\sum L_i \delta_i$, each of which can be bounded
independently.  The certificate (Theorem~\ref{thm:tube}) chains
the two: factorization controls the commands, contraction
propagates commands to states.

\subsection{Per-factor Lipschitz score}
\label{app:assumption-lip-score}

\begin{assumption}[Per-factor Lipschitz score]\label{asm:lip-score}
  For each factor $i \in \{1, \ldots, K\}$,
  \begin{equation}\label{eq:lip-score}
    \norm{s_\theta(\cdot, z_i) - s_\theta(\cdot, z_i')}
    \leq L_i \norm{z_i - z_i'}
  \end{equation}
  for all $z_i, z_i' \in \calZ_i$ and all
  $(a_\sigma, \sigma, o)$, where $L_i$ is the Lipschitz
  constant of the $i$-th conditioning branch.
\end{assumption}

This is a mild regularity property of the trained network
rather than a substantive task hypothesis.  For finite discrete
factor spaces $\calZ_i$ (such as the eight tracks and three
gate-opening sizes used in our experiments) the
right-hand side of~\eqref{eq:lip-score} is automatically finite:
$L_i$ is the maximum-over-pairs ratio
$\max_{z_i \neq z_i'} \norm{s_\theta(\cdot, z_i) -
s_\theta(\cdot, z_i')} / \norm{z_i - z_i'}$ and is computed
exactly by evaluating all factor pairs.  For continuous factor
spaces the bound follows from the network architecture: each
$z_i$ enters the network through a learned embedding followed
by Lipschitz operations (linear layers with bounded weights,
layer normalization, and standard nonlinearities), whose
composition is Lipschitz in the input.  What matters
operationally is the \emph{magnitude} of $L_i$, since this
controls the per-factor term $L_i \delta_i$ in the tube radius
of Theorem~\ref{thm:tube}; we measure $L_1, L_2$ directly in
Table~\ref{tab:exp-DL}.

\subsection{Identifiability via factor dropout}
\label{app:identifiability}

The decomposition $s_{\mathrm{comp}} = s + \sum_i \Delta_i$
(Definition~\ref{def:composed}) is not uniquely determined by
training data alone.  For any constants $c_i \in \R$ with
$\sum_i c_i = 0$, the substitution
$\tilde s := s + \sum_i c_i$,
$\tilde\Delta_i := \Delta_i - c_i$, leaves the composed score
$s + \sum_i \Delta_i$ unchanged but redistributes mass between
the unconditional baseline and the per-factor corrections.
Without an additional constraint, a learned $s_\theta$ could
absorb arbitrary content into $s$ that ought to live in
$\Delta_i$, in which case Theorem~\ref{thm:decomp}'s
$2\sqrt{GM}$ bound on the true compositional gap would not
transfer to $s_\theta^{\mathrm{comp}}$.

\paragraph{Factor dropout pins the decomposition.}
We resolve this in training by independently replacing each
factor conditioning $z_i$ with a learned null token
$\varnothing_i$ with probability $p_{\mathrm{drop}}$, so that
during training the network sees all $2^K$ subsets of factors.
The unconditional score $s_\theta(\cdot, \varnothing_1, \ldots,
\varnothing_K)$ is then trained to match the score of
$p(a \mid o)$ marginalized over all factors (the all-null
behavior), and each single-factor score
$s_\theta(\cdot, \varnothing_1, \ldots, z_i, \ldots,
\varnothing_K)$ is trained to match the score of
$p(a \mid o, z_i)$, marginalized over the others.  This pins
both $s$ and each $\Delta_i = s(\cdot, z_i) - s(\cdot)$ to
their information-theoretic targets and removes the constant
ambiguity above.  The dropout probability $p_{\mathrm{drop}}$
plays the same role as in classifier-free
guidance~\citep{ho2021classifierfree}; we use $p_{\mathrm{drop}} = 0.1$
throughout.  Without factor dropout, the per-factor corrections
$\Delta_i$ measured from $s_\theta$ are not the targets to which
Theorem~\ref{thm:decomp} applies, and the $2\sqrt{GM}$ bound
would not transfer to the learned model.

\subsection{Proof of Theorem~\ref{thm:decomp}
            (decomposition error bound)}
\label{app:proof:decomp}

\begin{theorem*}[Decomposition error bound,
                 Theorem~\ref{thm:decomp} restated]
  Suppose Assumption~\ref{asm:approx-indep} holds and \sayan{for every
  $\sigma \in [0, \sigma_{\max}]$,} the interaction log-ratio
  \sayan{$g_\sigma$}~\eqref{eq:g} is twice differentiable in
  \sayan{$a_\sigma$} with
  \sayan{$\norm{\nabla_{a_\sigma}^2 g_\sigma(z, a_\sigma, o)}_{\mathrm{op}} \leq M$}
  uniformly \sayan{in $\sigma, a_\sigma, o, z$}.
  Then
  $\norm{s(a_\sigma, \sigma, o, z)
         - s_{\mathrm{comp}}(a_\sigma, \sigma, o, z)}
   \leq 2\sqrt{G M}$
  for all $a_\sigma$, $\sigma$, $o$, and
  $z = (z_1, \ldots, z_K)$.
\end{theorem*}

\begin{proof}
  \emph{Step 1 (gradient bound on \sayan{$g_\sigma$}).}
  \sayan{Fix $\sigma \in [0, \sigma_{\max}]$ and $(z, a_\sigma, o)$;
  let $u := \nabla_{a_\sigma} g_\sigma(z, a_\sigma, o)$;} assume
  $\norm{u} > 0$ (the bound is trivial otherwise).  With
  $v := u / \norm{u}$ and the slice
  \sayan{$\phi(t) := g_\sigma(z, a_\sigma + tv, o)$},
  \begin{align*}
    \phi(t)
      &= \phi(0) + t\,\phi'(0) + \int_0^t (t-s)\,\phi''(s)\,ds
      && \text{[Taylor, integral remainder]} \\
      &\geq \phi(0) + t\,\norm{u} - \tfrac{1}{2}M t^2
      && \text{[$|\phi''| \leq M$;\; $\phi'(0) = \norm{u}$]} \\
      &\geq -G + t\,\norm{u} - \tfrac{1}{2}M t^2.
      && \text{[$|\phi(0)| \leq G$ by
               Assumption~\ref{asm:approx-indep}]}
  \end{align*}
  Since $\phi(t) \leq G$ as well, $\norm{u} \leq 2G/t + Mt/2$ for
  every $t > 0$; minimizing at $t^* = 2\sqrt{G/M}$ gives
  \sayan{$\norm{\nabla_{a_\sigma} g_\sigma(z, a_\sigma, o)} \leq 2\sqrt{G\,M}$
  uniformly in $\sigma, z, a_\sigma, o$.}

  \emph{Step 2 (score decomposition identity).}
  For all $(a_\sigma, \sigma, o, z)$, suppress these arguments
  with $(\cdot)$ for brevity.  For any conditioning variable
  $c$, Bayes' rule gives $p_\sigma(a_\sigma \mid o, c)
    = p_\sigma(a_\sigma \mid o)
      \cdot p_\sigma(c \mid a_\sigma, o) / p_\sigma(c \mid o)$.
  Since $p_\sigma(c \mid o)$ does not depend on $a_\sigma$,
  taking $\nabla_{a_\sigma}\log$ yields the score decomposition
  identity
  \begin{equation}\label{eq:bayes-score}
    s(\cdot, c)
      = s(\cdot)
        + \nabla_{a_\sigma} \log p_\sigma(c \mid a_\sigma, o).
  \end{equation}
  Then:
  \begin{align*}
    s(\cdot, z) - s_{\mathrm{comp}}
      &= \bigl[s(\cdot)
         + \nabla_{a_\sigma} \log p_\sigma(z \mid a_\sigma, o)
         \bigr] - s_{\mathrm{comp}}
      && \text{[\eqref{eq:bayes-score} with $c = z$]} \\
      &= \nabla_{a_\sigma} \log p_\sigma(z \mid a_\sigma, o)
         - \sum_{i=1}^{K} \Delta_i
      && \text{[Def.~\ref{def:composed}; cancel $s(\cdot)$]} \\
      &= \nabla_{a_\sigma} \log p_\sigma(z \mid a_\sigma, o)
         - \sum_{i=1}^{K}
           \nabla_{a_\sigma} \log p_\sigma(z_i \mid a_\sigma, o)
      && \text{[Def.~\ref{def:corrections};
               \eqref{eq:bayes-score} with $c = z_i$]} \\
      &= \nabla_{a_\sigma}\, \sayan{g_\sigma}(z, a_\sigma, o).
      && \text{[interaction log-ratio~\eqref{eq:g}\sayan{, at noise level $\sigma$}]}
  \end{align*}
  Taking norms and applying Step~1,
  $\norm{s(\cdot, z) - s_{\mathrm{comp}}}
    = \norm{\nabla_{a_\sigma} \sayan{g_\sigma}(z, a_\sigma, o)}
    \leq 2\sqrt{G\,M}$.
\end{proof}

\begin{remark}[Lifting the bounded-interaction assumption to noise levels]
\label{rem:noise-lift}
Assumption~\ref{asm:approx-indep} and Theorem~\ref{thm:decomp}'s
curvature condition are stated on the noised log-ratio $g_\sigma$
because the proof above invokes Bayes' rule on $p_\sigma$ and reads
the gap $s(\cdot, z) - s_{\mathrm{comp}}$ as $\nabla_{a_\sigma} g_\sigma$
at every noise level the reverse ODE traverses.  A clean-action
bound $|g_0| \leq G$ does not in general imply $|g_\sigma| \leq G$
for $\sigma > 0$: Gaussian convolution alters the conditional
densities $p_\sigma(z \mid a_\sigma, o)$ relative to
$p(z \mid a, o)$, and $g_\sigma$ interpolates between $g_0$
(at $\sigma \to 0$) and a fully marginalized log-ratio over
observations alone (at $\sigma \to \sigma_{\max}$, where action
information is washed out).  Two grounds for why the lifted bound
is realistic in our setting:

\textbf{(i) Closed form for deterministic experts.}
Section~\ref{sec:certificate} establishes
$p(a_0 \mid o, z) = \delta(a_0 - a_0^\star(o, z))$, so
$p_\sigma(a_\sigma \mid o, z) = \mathcal{N}(\alpha_\sigma a_0^\star(o, z),
\sigma^2 I)$ and the noised posterior $p_\sigma(z \mid a_\sigma, o)$
is a softmax over Gaussian likelihoods centered at the per-task
expert actions $\{\alpha_\sigma a_0^\star(o, z)\}_z$.  The log-ratio
$g_\sigma$ then has a closed form, and uniform boundedness in
$\sigma$ reduces to a finite-diameter condition on the expert-action
set $\{a_0^\star(o, z) : z \in \prod_i \calZ_i\}$ for each $o$,
a standing requirement of bounded-action policies.

\textbf{(ii) Empirical evidence.}
The decomposition error
$\varepsilon_D = \norm{s_{\mathrm{joint}} - s_{\mathrm{comp}}}$
measured on the trained network (Table~\ref{tab:exp-DL}) estimates
$\norm{\nabla_{a_\sigma} g_\sigma}$ at the noised
$(a_\sigma, \sigma, o)$ encountered during sampling
(Proposition~\ref{prop:residual}), together with the network's own score
error.  Theorem~\ref{thm:decomp} bounds that quantity by $2\sqrt{GM}$, so a
small measured $\varepsilon_D$ is consistent with the lifted assumption
holding on the operating distribution but does not bound $2\sqrt{GM}$ from
above.  On training tasks the mean is $0.45$; on held-out tasks it is $1.08$.
\end{remark}

\subsection{Proof of Proposition~\ref{prop:residual} (residual identity)}
\label{app:proof:residual}

\begin{proposition*}[Proposition~\ref{prop:residual} restated]
  If the scores in Definition~\ref{def:corrections} are exact, that is
  $s(a_\sigma, \sigma, o, z_i) = \nabla_{a_\sigma} \log p_\sigma(a_\sigma \mid o, z_i)$
  and $s(a_\sigma, \sigma, o) = \nabla_{a_\sigma} \log p_\sigma(a_\sigma \mid o)$,
  then for all $a_\sigma, \sigma, o, z$,
  $s(a_\sigma, \sigma, o, z) - s_{\mathrm{comp}}(a_\sigma, \sigma, o, z)
    = \nabla_{a_\sigma}\, g_\sigma(z, a_\sigma, o)$.
\end{proposition*}

\begin{proof}
  By Bayes' rule,
  $\log p_\sigma(a_\sigma \mid o, z) = \log p_\sigma(z \mid a_\sigma, o)
   + \log p_\sigma(a_\sigma \mid o) - \log p_\sigma(z \mid o)$,
  and the last term does not depend on $a_\sigma$.  Hence
  $s(\cdot, z) = \nabla_{a_\sigma} \log p_\sigma(z \mid a_\sigma, o) + s(\cdot)$,
  and in the same way
  $\Delta_i = \nabla_{a_\sigma} \log p_\sigma(z_i \mid a_\sigma, o)$.
  Subtracting~\eqref{eq:composed} leaves $\nabla_{a_\sigma}$ of the logarithm
  of the ratio in~\eqref{eq:g}.
\end{proof}

\subsection{Proof of Lemma~\ref{lem:ode-sens} (ODE sensitivity)}
\label{app:proof:ode-sens}

\begin{lemma*}[ODE sensitivity, Lemma~\ref{lem:ode-sens} restated]
  Let $s, s' : \calA \times [0, \sigma_{\max}] \to \calA$ be two
  score fields with (a)~$\norm{s - s'}_\infty \leq \varepsilon$
  and (b)~$s$ Lipschitz in its first argument with rate
  $L_a(\sigma)$.  Let $\{a_\sigma\}, \{a_\sigma'\}$ be denoising
  paths generated by the reverse ODE~\eqref{eq:reverse-ode} from
  a shared initial noise $a_{\sigma_{\max}} = a_{\sigma_{\max}}'$.
  Then $\norm{a_0 - a_0'} \leq C_{\mathrm{ode}}\,\varepsilon$,
  with $C_{\mathrm{ode}}$ given by~\eqref{eq:ode-bound}.
\end{lemma*}

\begin{proof}
  Define the error along the denoising path:
  $e(\sigma) = a_\sigma - a_\sigma'$.  By the shared initial noise,
  $e(\sigma_{\max}) = 0$.  The reverse ODE~\eqref{eq:reverse-ode}
  gives
  \[
    \frac{de}{d\sigma}
      = f_{\mathrm{dr}}(\sigma)\, e(\sigma)
        + \tilde{g}(\sigma)\bigl[
            s(a_\sigma, \sigma)
            - s'(a_\sigma', \sigma)
          \bigr].
  \]
  Split the score difference:
  \begin{align*}
    s(a_\sigma, \sigma) - s'(a_\sigma', \sigma)
      &= \underbrace{
           \bigl[s(a_\sigma, \sigma) - s(a_\sigma', \sigma)\bigr]
         }_{\text{same field, different points:\;}
           \leq\, L_a(\sigma)\, \norm{e}}
       + \underbrace{
           \bigl[s(a_\sigma', \sigma) - s'(a_\sigma', \sigma)\bigr]
         }_{\text{different fields, same point:\;}
           \leq\, \varepsilon}.
  \end{align*}
  Setting
  $\mu(\sigma) = |f_{\mathrm{dr}}(\sigma)|
                 + |\tilde{g}(\sigma)|\, L_a(\sigma)$:
  \[
    \Big\lVert\frac{de}{d\sigma}\Big\rVert
      \leq \mu(\sigma)\, \norm{e}
           + |\tilde{g}(\sigma)|\, \varepsilon.
  \]
  This is a linear differential inequality with forcing term,
  integrated backward from $\sigma_{\max}$ (where $\norm{e} = 0$)
  to $\sigma = 0$.  Gr\"onwall's inequality gives
  \[
    \norm{e(\sigma)}
      \leq \int_\sigma^{\sigma_{\max}}
           |\tilde{g}(\sigma')|\, \varepsilon\,
           \exp\!\Bigl(
             \int_\sigma^{\sigma'} \mu(\sigma'')\, d\sigma''
           \Bigr)\, d\sigma'.
  \]
  Evaluating at $\sigma = 0$ yields
  $\norm{a_0 - a_0'} = \norm{e(0)}
  \leq C_{\mathrm{ode}} \cdot \varepsilon$.
\end{proof}

\begin{remark}
  Unlike the standard Gr\"onwall setting (same ODE, different
  initial conditions), here both denoising paths start from the
  \emph{same} initial noise ($e(0) = 0$).  The divergence is
  driven entirely by the score-field mismatch $\varepsilon$,
  which enters as a forcing term.  The Lipschitz constant $L_a$
  then amplifies this forcing through feedback: the two paths
  drift apart, causing the \emph{same} score field to evaluate
  differently at the two locations, which further increases the
  drift.  The exponential in
  $C_{\mathrm{ode}}$~\eqref{eq:ode-bound} reflects this feedback
  amplification.
\end{remark}

\subsection{Path-dependent LTV sensitivity (Lemma~\ref{lem:ltv-sens})}
\label{app:proof:ltv-sens}

\paragraph{DDIM linearization.}
Discretize the reverse ODE~\eqref{eq:reverse-ode} with $K$ DDIM
steps.  Let $\bar a_k$ ($k = 0, 1, \ldots, K$) be the nominal
denoising trajectory from $a^\star$ at noise levels
$\sigma_0 = \sigma_{\max} > \sigma_1 > \cdots > \sigma_K = 0$.
Each step is
\begin{equation}\label{eq:ddim-linmap}
  a_{k+1} \;=\; c_1(k)\, a_k \;+\; c_2(k)\, \sayan{\hat\varepsilon_\theta(a_k, \sigma_k)},
\end{equation}
where $c_1, c_2$ are scheduler-determined coefficients (for
DDIM: $c_1(k) = \sqrt{\bar\alpha_{k+1}/\bar\alpha_k}$,
$c_2(k) = \sqrt{1 - \bar\alpha_{k+1}}
        - \sqrt{\bar\alpha_{k+1}(1 - \bar\alpha_k)/\bar\alpha_k}$,
in the notation of~\cite{song2021denoising}).
\sayan{The eps-prediction $\hat\varepsilon_\theta$ and the score
$s$ are related by $\hat\varepsilon_\theta(x, \sigma) = -\sqrt{1 - \bar\alpha_\sigma}\, s(x, \sigma)$,
so the score-mismatch hypothesis $\norm{s - s'}_\infty \leq \varepsilon$
of Lemma~\ref{lem:ode-sens} transfers to
$\norm{\hat\varepsilon_\theta - \hat\varepsilon'_\theta}_\infty \leq \varepsilon$
(since $\sqrt{1 - \bar\alpha_\sigma} \leq 1$).  We state and measure
the LTV bound below in $\hat\varepsilon_\theta$-form throughout, matching
standard DDIM implementations.}
Linearizing around the nominal yields per-step Jacobians
\begin{equation}\label{eq:M-step}
  M_k \;=\; c_1(k)\, I \;+\; c_2(k)\, J_k,
  \qquad J_k := \sayan{\partial \hat\varepsilon_\theta / \partial a}\bigr|_{\bar a_k, \sigma_k},
\end{equation}
and the discrete state-transition matrix from step $k$ to $K$ is
\begin{equation}\label{eq:Phi-def}
  \Phi_{k,K} \;:=\; M_{K-1} M_{K-2} \cdots M_k,
  \qquad \Phi_{K,K} := I.
\end{equation}

\begin{lemma}[Path-dependent LTV sensitivity]\label{lem:ltv-sens}
  Under the assumptions of Lemma~\ref{lem:ode-sens} and the
  shared-seed condition
  $a_{\sigma_{\max}} = a'_{\sigma_{\max}} = a^\star$, and to
  first order in $\varepsilon$ along the nominal path
  $\{\bar a_k\}_{k=0}^K$,
  \begin{equation}\label{eq:ltv-bound}
    \norm{a_K - a_K'}
      \;\leq\; C_{\mathrm{ode}}^{\mathrm{LTV}}(\bar a)\,\varepsilon
      \;+\; O(\varepsilon^2),
    \qquad
    C_{\mathrm{ode}}^{\mathrm{LTV}}(\bar a)
      \;=\; \sum_{k=0}^{K-1} \abs{c_2(k)}\,
              \norm{\Phi_{k+1,K}}_2.
  \end{equation}
\end{lemma}

\sayan{(\emph{Indexing convention.})  The $K$ reverse-sampling steps are indexed so that $k = 0$ is the noisy start ($\sigma_0 = \sigma_{\max}$) and $k = K$ is the clean end ($\sigma_K = 0$), matching the order in which DDIM evaluates them.  The clean-end action error $\norm{a_K - a_K'}$ in this lemma therefore corresponds to $\norm{a_0 - a_0'}$ in the continuous-time Lemma~\ref{lem:ode-sens}, where $a_0 = a_{\sigma = 0}$ is the clean action by the standard convention.}

\begin{proof}
  Let $e_k := a_k - a_k'$, so $e_0 = 0$ by the shared initial
  condition.  Define the per-step \sayan{eps-prediction
  forcing $\Delta\hat\varepsilon_k := \hat\varepsilon_\theta(a_k', \sigma_k)
  - \hat\varepsilon'_\theta(a_k', \sigma_k)$} and
  $b_k := c_2(k)\,\sayan{\Delta\hat\varepsilon_k}$;
  \sayan{via the bridge $\hat\varepsilon_\theta = -\sqrt{1-\bar\alpha_\sigma}\, s$,}
  the score-field mismatch assumption
  $\norm{s - s'} \leq \varepsilon$ gives
  $\sayan{\norm{\Delta\hat\varepsilon_k} \leq \sqrt{1-\bar\alpha_k}\,\varepsilon \leq \varepsilon}$,
  and hence $\norm{b_k} \leq \abs{c_2(k)}\,\varepsilon$.
  \begin{align*}
    e_{k+1}
      &= M_k\, e_k + b_k + O(\norm{e_k}^2)
      && \text{[linearize~\eqref{eq:ddim-linmap} around $\bar a_k$]} \\
    e_K
      &= \sum_{k=0}^{K-1} \Phi_{k+1,K}\, b_k \;+\; O(\varepsilon^2)
      && \text{[discrete variation of constants from $e_0 = 0$;\;} \\
               &&& \ \  \text{$b_k = O(\varepsilon)$ absorbs the nonlinear remainder]} \\
    \norm{e_K}
      &\leq \sum_{k=0}^{K-1} \norm{\Phi_{k+1,K}}_2\,\norm{b_k}
            + O(\varepsilon^2)
      && \text{[triangle inequality, then submultiplicativity of $\norm{\cdot}_2$]} \\
      &\leq \sum_{k=0}^{K-1} \abs{c_2(k)}\,\norm{\Phi_{k+1,K}}_2\,\varepsilon
            + O(\varepsilon^2)
      && \text{[apply $\norm{b_k} \leq \abs{c_2(k)}\,\varepsilon$]}.
  \end{align*}
\end{proof}

\begin{remark}[Gr\"onwall as a further relaxation]
  \label{rem:ltv-grn}
  Applying the submultiplicativity inequality
  $\norm{\Phi_{k+1,K}}_2 \leq \prod_{j=k+1}^{K-1} \norm{M_j}_2$
  in~\eqref{eq:ltv-bound} together with the triangle inequality
  $\norm{M_j}_2 \leq \abs{c_1(j)} + \abs{c_2(j)}\, L_a(\sigma_j)$
  yields the discrete analog of the continuous Gr\"onwall bound
  in Lemma~\ref{lem:ode-sens}; passing to the continuous-time
  limit recovers $C_{\mathrm{ode}}^{\mathrm{ana}}$
  in~\eqref{eq:ode-bound} exactly.  Hence
  $C_{\mathrm{ode}}^{\mathrm{LTV}}(\bar a) \leq C_{\mathrm{ode}}^{\mathrm{ana}}$,
  with strict inequality when (i) inter-step cancellation occurs
  (matrix product tighter than product of norms) and (ii) Jacobian
  structure is not captured by the scalar Lipschitz constant
  (discrete $\sigma_{\max}$ tighter than $L_a(\sigma_j)$).  Both
  effects are typically present: Section~\ref{sec:exp-v6} reports
  $C_{\mathrm{ode}}^{\mathrm{ana}} \approx 10^5$--$10^6$ (vacuous)
  versus $C_{\mathrm{ode}}^{\mathrm{LTV}} \approx 2$--$10$ on the
  $K = 50$ DDIM chain of our factored model, five to six orders
  of magnitude tighter.
\end{remark}

\begin{remark}[Remaining gap: nonlinear manifold contraction]
  \label{rem:manifold-gap}
  Lemma~\ref{lem:ltv-sens} is a linearized worst-case bound:
  it (i) drops the $O(\varepsilon^2)$ Taylor remainder and
  (ii) assumes the score-error direction $\Delta s_k$ aligns
  with $\Phi_{k+1,K}$'s top singular direction at every step.
  The linearization drop is an unavoidable feature of any
  sensitivity analysis; it is controlled as long as
  $\varepsilon \cdot C_{\mathrm{ode}}^{\mathrm{LTV}}$ remains
  small compared to the local radius of linearization.  The
  direction-alignment assumption is \emph{conservative}:
  empirically, $\Delta s_k$ does not align with
  $\Phi_{k+1,K}$'s worst direction across steps, and the bound
  overestimates the measured action gap by a factor of order
  $10^{1}$--$10^{2}$ (Section~\ref{sec:exp-v6}).
  Closing this gap would require either distributional
  assumptions on $\Delta s_k$ (stochastic Gr\"onwall) or an
  analysis that exploits the \emph{manifold-contractive}
  character of the learned score field: empirically, the
  denoising ODE is more contractive along the data manifold
  than its Jacobian predicts, because the score field itself
  points back toward the data manifold.  Neither refinement
  is pursued in this paper; we report
  $C_{\mathrm{ode}}^{\mathrm{LTV}}$ as the best worst-case
  bound achievable by norm-based path analysis and separately
  report the empirical amplification in
  Section~\ref{sec:exp-v6}.
\end{remark}

\subsection{Proof of Theorem~\ref{thm:tube} (factored tube)}
\label{app:proof:tube}

\begin{theorem*}[Factored trajectory tube,
                 Theorem~\ref{thm:tube} restated]
  Under
  Assumptions~\ref{asm:approx-indep},
  \ref{asm:lip-score}, and
  \ref{asm:contraction},
  and either the scalar Lipschitz condition $L_a(\sigma)$ of
  Lemma~\ref{lem:ode-sens} (hypothesis~(a)) or the
  path-Jacobian regularity of Lemma~\ref{lem:ltv-sens} under
  the shared-seed condition
  $a_{\sigma_{\max}} = a'_{\sigma_{\max}} = a^\star$
  (hypothesis~(b)), the
  closed-loop state error satisfies the trajectory-tube
  bound~\eqref{eq:tube} with
  $\varepsilon_s \leq 2\sqrt{GM} + \sayan{(2K-1)\eta}
   + \sum_{i=1}^{K} L_i \delta_i$.
\end{theorem*}

\begin{proof}
  Define $e_t = x_t - x_t^{\mathrm{nom}}$.

  \emph{Step 1 (score perturbation).}
  The gap between the true joint score at the nominal factors
  and the composed learned score at the perturbed factors
  decomposes via three triangle-inequality steps,
  \sayan{routed through the composed scores at the nominal task
  $z^{\mathrm{nom}}$ so that Assumption~\ref{asm:lip-score} is
  invoked on the trained network rather than on the unlearned
  true score}:
  \begin{align*}
    &\norm{s(\cdot, z^{\mathrm{nom}})
           - s_\theta^{\mathrm{comp}}(\cdot, z)} \\
    &\quad\leq
      \underbrace{
        \norm{s(\cdot, z^{\mathrm{nom}})
              - s^{\mathrm{comp}}(\cdot, z^{\mathrm{nom}})}
      }_{\leq\, 2\sqrt{GM} \;\text{(Thm~\ref{thm:decomp})}}
    + \underbrace{
        \norm{s^{\mathrm{comp}}(\cdot, z^{\mathrm{nom}})
              - s_\theta^{\mathrm{comp}}(\cdot, z^{\mathrm{nom}})}
      }_{\leq\, (2K-1)\eta
        \;\text{(equation~\eqref{eq:composed-kappa})}} \\
    &\quad\quad
    + \underbrace{
        \norm{s_\theta^{\mathrm{comp}}(\cdot, z^{\mathrm{nom}})
              - s_\theta^{\mathrm{comp}}(\cdot, z)}
      }_{\leq\, \sum_i L_i \delta_i
        \;\text{(Asm~\ref{asm:lip-score})}}.
  \end{align*}
  \sayan{The third term expands as
  $s_\theta^{\mathrm{comp}}(\cdot, z^{\mathrm{nom}})
   - s_\theta^{\mathrm{comp}}(\cdot, z)
   = \sum_{i=1}^{K} \bigl[
       s_\theta(\cdot, z_i^{\mathrm{nom}})
       - s_\theta(\cdot, z_i)
     \bigr]$
  because both composed scores share the unconditional baseline
  $s_\theta(\cdot)$; Assumption~\ref{asm:lip-score} on each
  per-factor branch then gives
  $\norm{s_\theta(\cdot, z_i^{\mathrm{nom}}) - s_\theta(\cdot, z_i)}
   \leq L_i \norm{z_i^{\mathrm{nom}} - z_i} \leq L_i \delta_i$,
  and the triangle inequality yields the displayed bound.}
  Hence
  $\varepsilon_s \leq 2\sqrt{GM} + \sayan{(2K-1)\eta}
   + \sum_i L_i \delta_i$.

  \emph{Step 2 (action perturbation).}
  Under hypothesis~(a), Lemma~\ref{lem:ode-sens} gives
  $\norm{a_t - a_t^{\mathrm{nom}}}
   \leq C_{\mathrm{ode}}^{\mathrm{ana}} \varepsilon_s$
  uniformly over tasks.  Under hypothesis~(b),
  Lemma~\ref{lem:ltv-sens} gives the tighter per-task bound
  $\norm{a_t - a_t^{\mathrm{nom}}}
   \leq C_{\mathrm{ode}}^{\mathrm{LTV}}(\bar a^{(z)})\,
        \varepsilon_s \sayan{\,+\, O(\varepsilon_s^2)}$
  along the nominal denoising trajectory $\bar a^{(z)}$
  \sayan{(the $O(\varepsilon_s^2)$ remainder inherited from
  Lemma~\ref{lem:ltv-sens}'s discrete linearization is controlled
  by Remark~\ref{rem:manifold-gap})}.

  \emph{Step 3 (contraction).}
  From \eqref{eq:contraction} with
  $d = B_\kappa C_{\mathrm{ode}}\, \varepsilon_s + w$:
  $\norm{e_{t+1}} \leq \lambda\, \norm{e_t} + d$.
  Unrolling: $\norm{e_t} \leq \lambda^t \delta_0 + d/(1-\lambda)$
  \sayan{(under hypothesis~(b), the $O(\varepsilon_s^2)$ from
  Step~2 propagates additively through the per-step disturbance
  and the unrolled bound carries an $O(\varepsilon_s^2)$ correction
  in the steady-state term)}.
\end{proof}

\section{Full-Race Experiments: Detailed Results}
\label{app:full-race}

This appendix collects the per-task numbers, ablations, and
certificate diagnostics behind the headline §\ref{sec:exp-v6}
results.  Subsections are organized roughly in the order of the
main-text claims: tracks and reference speeds (§\ref{app:fr:tracks}),
per-task and per-track outcomes (§\ref{app:fr:per-task}), the
composed-vs-joint comparison (§\ref{app:fr:cvj}), score-field
diagnostics (§\ref{app:fr:decomp-diag}), the $K$-network baseline's
training pipeline (§\ref{app:fr:knet}), the DDIM step sweep
(§\ref{app:fr:ddim-sweep}), the
path-dependent LTV bound
(§§\ref{app:fr:norm-bounds}--\ref{app:fr:cont-limit}), and the
$10$-seed robustness study (§\ref{app:fr:seed}) and residual-gap
diagnostic (§\ref{app:fr:residual-gap}).

\subsection{Training hyperparameters and compute}
\label{app:fr:training}

Both the state-based full-race policy
(Section~\ref{sec:exp-v6}) and the vision-based single-gate
policy (Appendix~\ref{sec:exp-sg}) share a common training
recipe except where noted.  Diffusion uses a DDPM cosine
schedule with $100$ noise steps for training and $50$-step
DDIM at inference, AdamW with learning rate $10^{-4}$, batch
size $32$, and $1500$ epochs.  All training and evaluation
runs use a workstation with one NVIDIA GeForce RTX~4080~SUPER
($16$\,GB VRAM), an AMD Ryzen~9 7900X CPU
($12$ cores, $24$ threads), and $64$\,GB system RAM, running
CUDA~12.1 with PyTorch~2.4.1.  Wall-clock time: state-based
training is $\sim\!30$\,min per model; vision-based training
runs $\sim\!2$\,h per DAgger round and $\sim\!50$ epochs total;
the $12$-network composition baseline of
Section~\ref{sec:exp-v6} (Appendix~\ref{app:fr:knet})
totals $\sim\!40$\,min across all networks.  All path-Jacobian
and seed-sweep certificate measurements (Tables~\ref{tab:exp-ltv}
and~\ref{tab:exp-seed-ci}) run on the same single GPU.

\subsection{Race tracks and reference speed profiles}
\label{app:fr:tracks}\label{app:tracks}

The eight race tracks share a common gate layout.  Maximum
feasible speeds are determined by binary search: for each
candidate speed, the $\mathrm{SE}(3)$ expert flies the cubic-spline
reference and every gate is checked for passage with the actual
half-width $r(z_2)$.  A combination is declared infeasible if no
candidate speed (down to a small lower bound) clears every gate.

\begin{table}[h]
\centering
\small
\caption{Race tracks and maximum feasible speeds (m/s).
         ``---'' indicates infeasible at that gate size.}
\label{tab:tracks-summary}
\begin{tabular}{lrcccc}
  \toprule
  Track & Gates & Narrow & Standard & Wide & Description \\
  \midrule
  race1 & 8  & --- & 2.1 & 2.4 & Complex, U-turns \\
  race2 & 9  & --- & 2.1 & 2.4 & Corkscrew \\
  race3 & 7  & --- & 2.1 & 2.5 & Right-to-left sweep \\
  race4 & 10 & --- & 3.3 & 3.5 & Long winding \\
  race5 & 5  & \sayan{2.2} & 3.6 & 3.9 & Short loop \\
  race6 & 3  & 2.3 & 4.1 & 4.6 & Simple run \\
  race7 & 2  & 8.2 & 10.0 & 10.0 & Height change \\
  race8 & 2  & 7.0 & 7.9 & 7.9 & U-turn \\
  \bottomrule
\end{tabular}
\end{table}

The full task space has $8 \times 3 = 24$ tasks; $20$ are
feasible.  \sayan{Races~1--4} with narrow gates exceed the
$\mathrm{SE}(3)$ controller's tracking precision and are
omitted.  Race~6 has the strongest gate-size speed
differentiation ($2.0\times$ ratio between narrow and wide);
races~7--8 are nearly speed-saturated, so the wide and standard
maxima coincide.

\begin{figure}[p]
  \centering
    \includegraphics[width=0.9\textwidth]{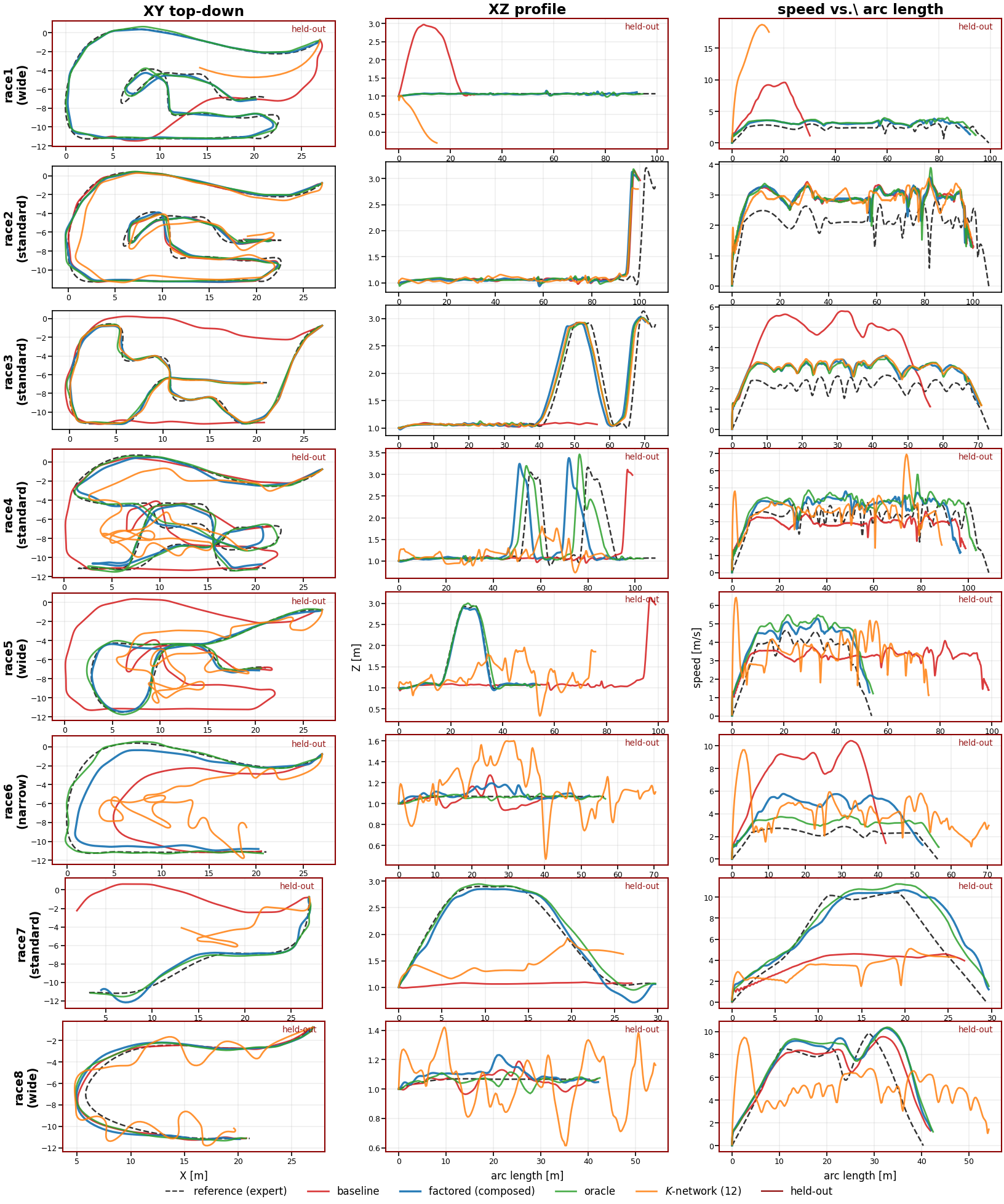}
  \caption{\small \textbf{Per-race closed-loop trajectories with all five methods
           overlaid.}  One row per track; columns are XY top-down (geometric
           route), XZ side profile (height), and speed vs.\ arc length.
           For each race we pick the held-out joint pair where one exists
           (red border, ``held-out'' badge), else \emph{standard}
           (race2, race3 have no held-out combo).  Methods: expert
           reference (dashed black), baseline (red), factored composed
           (blue, ours), oracle (green), and the $K$-network composition
           baseline (orange).  The factored policy tracks reference and
           oracle within ${\sim}0.1$\,m on every held-out row except
           (race6, narrow), where the decomposition error is $3.5\times$ the
           gate-size sensitivity, the highest ratio of any task
           (Table~\ref{tab:exp-DL}).  The
           baseline fails nearly every held-out gate}; the $K$-network
           composition produces erratic geometry and noisy speed profiles
           even on training tasks (race4\_standard, race8\_wide),
           corroborating Table~\ref{tab:exp-aggregate}.
  \label{fig:all-tracks}
\end{figure}

\subsection{Per-task and per-track results}
\label{app:fr:per-task}

\paragraph{Held-out task detail.}
Table~\ref{tab:exp-heldout} reports per-task gate passage on the
six held-out (track, gate-size) pairs at $N = 50$.  \sayan{Pass/fail
outcomes are invariant across the five DDIM seeds tested: different
seeds produce different DDIM-generated reference trajectories, and
the deterministic $\mathrm{SE}(3)$ controller plus plant tracks
each one to a different state trajectory, but the trajectory
variance across seeds stays within the gate's pass/fail margin
in every case. This is evidence of robustness, not of upstream
determinism.}

\begin{table}[h]
\centering
\small
\caption{Held-out tasks ($5$ seeds, $\sigma = 0\%$).
         Generated mean speed vs.\ reference from the speed map.}
\label{tab:exp-heldout}
\begin{tabular}{llcccc}
  \toprule
  & & \multicolumn{2}{c}{Gate passage} &
    \multicolumn{2}{c}{Speed [m/s]} \\
  \cmidrule(lr){3-4} \cmidrule(lr){5-6}
  Track & Size & Factored & Oracle & Gen. & Ref. \\
  \midrule
  race1 & wide & \textbf{100\%} & 100\% & \sayan{2.2} & 2.5 \\
  race4 & std  & \textbf{100\%} & 100\% & 3.0 & 3.4 \\
  race5 & wide & \textbf{100\%} & 100\% & 3.0 & 4.0 \\
  race6 & nar  & 0\% & \sayan{0\%} & \sayan{3.4} & 2.4 \\
  race7 & std  & \textbf{100\%} & 100\% & 5.8 & 9.5 \\
  race8 & wide & \textbf{100\%} & 100\% & \sayan{5.9} & 7.9 \\
  \midrule
  \multicolumn{2}{l}{Total}
    & \textbf{5/6 (83\%)} & \sayan{5/6} & & \\
  \bottomrule
\end{tabular}
\end{table}

\sayan{Both factored and oracle achieve $100\%$ on $5$ of $6$
held-out tasks.  Both fail (race6, narrow): the trained policy
generates $3.4$\,m/s for race6 (reference: $2.4$\,m/s), at the
edge of the SE(3) controller's ability to track slow speeds
through race6's geometry.  Since oracle (trained directly on
race6 narrow) also fails $0/3$, the limit is a control /
network-capacity constraint rather than a composition error;
factored inherits the same constraint.}

\paragraph{Speed fidelity.}
The gate-size score correction $\Delta_2$ should produce
monotonically increasing speeds as gate width increases.
Table~\ref{tab:exp-speed} confirms the ordering
narrow $<$ standard $\leq$ wide on every track where narrow is
feasible (with equality on race8, where standard and wide both
generate $6.0$\,m/s).  Generated speeds are conservative (below reference)
but the relative ordering is always correct.

\begin{table}[h]
\centering
\small
\caption{Generated mean speed (m/s) by gate size (factored model).
         Reference in parentheses.}
\label{tab:exp-speed}
\begin{tabular}{lccc}
  \toprule
  Track & Narrow & Standard & Wide \\
  \midrule
  race1--5 & --- & 2.0--3.0\,(2.2--3.7) & 2.3--3.3\,(2.5--4.0) \\
  race6 & \sayan{3.4}\,(2.4) & 3.5\,(4.2) & 3.8\,(4.8) \\
  race7 & 5.0\,(7.7) & 5.8\,(9.5) & 6.1\,(9.5) \\
  race8 & 5.3\,(7.0) & 6.0\,(7.9) & 6.0\,(7.9) \\
  \bottomrule
\end{tabular}
\end{table}

\paragraph{Per-track breakdown.}
Table~\ref{tab:exp-pertrack} shows gate passage for the factored
model per (track, gate-size), with held-out tasks in italic.
Wide-gate tasks pass $45/46$ and standard $44/46$; narrow remains
the bottleneck at $9/12$, with two of its three misses on race8 narrow at
$7$\,m/s through $0.3$\,m gates, at the dynamic limit of the SE(3) controller.
The four held-out misses fall on three tasks, (race6, narrow), (race8, wide)
and (race4, standard); the oracle trained on those tasks passes all three.

\begin{table}[h]
\centering
\small
\caption{Per-track gate passage (factored, composed, corrected scoring).
         \textit{Italic} = held-out.  ``---'' = infeasible.}
\label{tab:exp-pertrack}
\begin{tabular}{lcccc}
  \toprule
  Track & Narrow & Standard & Wide & Total \\
  \midrule
  race1 & --- & 8/8 & \textit{8/8} & 16/16 \\
  race2 & --- & 9/9 & 9/9 & 18/18 \\
  race3 & --- & 7/7 & 7/7 & 14/14 \\
  race4 & --- & \textit{9/10} & 10/10 & 19/20 \\
  race5 & 5/5 & 5/5 & \textit{5/5} & 15/15 \\
  race6 & \textit{1/3} & 3/3 & 3/3 & 7/9 \\
  race7 & 2/2 & \textit{2/2} & 2/2 & 6/6 \\
  race8 & 1/2 & 1/2 & \textit{1/2} & 3/6 \\
  \midrule
  & 9/12 & 44/46 & 45/46 & \textbf{98/104} \\
  \bottomrule
\end{tabular}
\end{table}

\subsection{Composed vs.\ joint inference}
\label{app:fr:cvj}

The headline §\ref{sec:exp-v6} numbers use \emph{composed}
inference, $s_{\mathrm{comp}} = s_\varnothing + \Delta_1(z_1)
+ \Delta_2(z_2)$ (Definition~\ref{def:composed}, requiring
$3$ forward passes per DDIM step).  An alternative is
\emph{joint} conditioning: pass both factor embeddings to the
shared network in a single forward pass.
Table~\ref{tab:exp-composed} compares the two modes on the same
factored model and the full $20$-task evaluation.

\begin{table}[h]
\centering
\small
\caption{Composed vs.\ joint inference (factored model, all
         $20$ feasible tasks).}
\label{tab:exp-composed}
\begin{tabular}{lccc}
  \toprule
  Inference & Gates & Rate & Error [m] \\
  \midrule
  Joint ($s(z_1, z_2)$ directly) & 101/104 & 97\% & 0.21 \\
  \textbf{Composed} ($s_\varnothing + \Delta_1 + \Delta_2$)
    & \textbf{98/104} & \textbf{94\%} & 0.24 \\
  \bottomrule
\end{tabular}
\end{table}

Composed inference trails joint readout by three gates on this
benchmark ($98$ against $101$ of $104$; $26$ against $29$ of $30$ held out),
and by more at sparser training coverage (Table~\ref{tab:exp-coverage}).
We do not claim that composed inference improves gate passage.  The ordering
is what the error accounting predicts: the composed score sums $2K-1$ network
evaluations and carries $(2K-1)\eta$ of score-matching error where joint
readout carries $\eta$, so thinning coverage, which raises $\eta$, widens the
gap.  Near-equality at the $14$-task coverage is what
Theorem~\ref{thm:decomp} predicts when the decomposition error is small
relative to the per-factor sensitivities, and at fixed coverage a widening gap
points to that error rather than to $\eta$; the interaction sweep of
Table~\ref{tab:exp-interaction} realizes exactly that.  What composed
inference buys is not accuracy but structure: a per-factor error budget for
the certificate (Theorem~\ref{thm:tube}), the null-token control used in
\S\ref{sec:exp-sg}, and a measurable residual, at a cost of $K{+}1$ forward passes
per denoising step and greater sensitivity to the step count
(§\ref{app:fr:ddim-sweep}).

\subsection{Decomposition and Lipschitz diagnostics}
\label{app:fr:decomp-diag}

Theorem~\ref{thm:decomp} bounds the decomposition gap by
$\norm{s_{\mathrm{joint}} - s_{\mathrm{comp}}} \leq 2\sqrt{GM}$,
where $G$ bounds the pointwise interaction log-ratio and $M$
bounds the score Hessian.  Rather than estimating $G$ and $M$
separately, we directly measure the left-hand side: the
decomposition error
$\varepsilon_D := \norm{s_{\mathrm{joint}} - s_{\mathrm{comp}}}$
on random (action, timestep) pairs across all training tasks.
Since Theorem~\ref{thm:decomp} gives
$\varepsilon_D \leq 2\sqrt{GM}$, a measured $\varepsilon_D$
\emph{lower}-bounds $2\sqrt{GM}$, and constrains the product
$GM$ from below rather than certifying it from above.  We report
$\varepsilon_D$ in place of $G$ and $M$ because it is the
quantity that actually propagates into the tube radius of
Theorem~\ref{thm:tube}, for which $2\sqrt{GM}$ is only a
conservative surrogate; measuring the gap directly avoids that
conservatism.

\paragraph{$\varepsilon_D$ is the interaction, not a proxy
for it.}  When factor dropout pins the conditional targets
(Appendix~\ref{app:identifiability}), so that $s_\varnothing$
matches $\nabla_a \log p(a \mid o)$ and each single-factor score
matches $\nabla_a \log p(a \mid o, z_i)$, the per-factor
correction of Definition~\ref{def:corrections} reduces to
$\Delta_i = \nabla_a \log p(z_i \mid a, o)$, since the remaining
term $\nabla_a \log p(z_i \mid o)$ vanishes for having no
dependence on $a$.  Subtracting the composed score from the joint
one, the unconditional term cancels and
\[
  s_{\mathrm{joint}} - s_{\mathrm{comp}}
  \;=\; \nabla_a \log
    \frac{p(z_1, \ldots, z_K \mid a, o)}
         {\prod_{i=1}^{K} p(z_i \mid a, o)} .
\]
The decomposition error is therefore the action-gradient of the
factors' total correlation given $(a, o)$, and vanishes
identically under Assumption~\ref{asm:approx-indep}.  So
$\varepsilon_D$ is a direct estimator of how badly that
assumption is violated, not a proxy for it, and the assumption is
falsifiable on any task where the score network can be queried.
This is what licenses the per-task use of $\varepsilon_D$ as a
predictor of composition failure.

\begin{table}[h]
\centering
\small
\caption{Per-task decomposition error $\varepsilon_D$ on the trained
         factored model, $50$ noised expert states per task drawn with a fixed
         seed, in the network's output units; $L_2$ is the gate-size
         sensitivity $\norm{s(z_2) - s(z_2')}$ measured on the same states,
         and the ratio $\varepsilon_D / L_2$ normalizes the residual by the
         weaker factor's own correction.  Gates: composed inference,
         corrected scoring.}
\label{tab:exp-DL}
\begin{tabular}{llcccc}
  \toprule
  Task & Split & $\varepsilon_D$ mean & $\varepsilon_D$ max & $\varepsilon_D / L_2$ & Gates \\
  \midrule
  race1\_standard & train & 0.49 & 2.75 & 0.46 & 8/8 \\
  race2\_standard & train & 0.35 & 1.11 & 0.96 & 9/9 \\
  race2\_wide     & train & 0.37 & 1.41 & 0.69 & 9/9 \\
  race3\_standard & train & 0.28 & 0.62 & 0.22 & 7/7 \\
  race3\_wide     & train & 0.33 & 1.68 & 0.20 & 7/7 \\
  race4\_wide     & train & 0.62 & 2.12 & 0.75 & 10/10 \\
  race5\_narrow   & train & 0.67 & 2.08 & 1.10 & 5/5 \\
  race5\_standard & train & 0.62 & 2.40 & 0.44 & 5/5 \\
  race6\_standard & train & 0.36 & 1.63 & 0.51 & 3/3 \\
  race6\_wide     & train & 0.46 & 2.61 & 0.40 & 3/3 \\
  race7\_narrow   & train & 0.34 & 1.15 & 0.34 & 2/2 \\
  race7\_wide     & train & 0.34 & 1.07 & 0.40 & 2/2 \\
  race8\_narrow   & train & 0.59 & 2.16 & 0.88 & 1/2 \\
  race8\_standard & train & 0.42 & 1.60 & 0.44 & 1/2 \\
  \midrule
  race1\_wide     & held-out & 0.62 & 2.53 & 0.79 & 8/8 \\
  race4\_standard & held-out & 1.06 & 4.76 & 1.50 & 9/10 \\
  race5\_wide     & held-out & 0.95 & 2.98 & 0.75 & 5/5 \\
  race6\_narrow   & held-out & 1.26 & 2.64 & \textbf{3.51} & 1/3 \\
  race7\_standard & held-out & 0.97 & 2.83 & 1.12 & 2/2 \\
  race8\_wide     & held-out & 1.62 & 5.89 & 1.07 & 1/2 \\
  \midrule
  mean, train    &  & 0.45 & & 0.6 & 72/74 \\
  mean, held-out &  & 1.08 & & 1.5 & 26/30 \\
  \bottomrule
\end{tabular}
\end{table}

Three readings.  The residual roughly doubles where the assumption is
extrapolated rather than interpolated, $0.45$ on training tasks against
$1.08$ held out.  The bound of Theorem~\ref{thm:decomp} is uniform in the
factor tuple and does not itself predict that gap; it is consistent with the
network-error component of $\varepsilon_D$ that we cannot separate out, and we
report it as an observation.  Normalized by the gate-size sensitivity, the
worst-performing task, (race6, narrow), is the extreme at $3.5$ against a
held-out mean of $1.5$.  The measurement supports no threshold: (race8, wide)
drops a gate at a ratio of $1.07$ and (race4, standard) passes $9$ of $10$ at
$1.50$, so what it establishes is that the assumption is checkable per task
before deployment, not that $\varepsilon_D$ predicts passage.

\subsection{$K$-network composition baseline}
\label{app:fr:knet}

The $K$-network row of Table~\ref{tab:exp-aggregate} reports a
PoCo / DSE-style baseline~\citep{Wang-RSS-24,patil2024composing}
adapted to our factor decomposition: rather than null-token
masking of one shared network, we train $1 + 8 + 3 = 12$
separate unconditional networks: an unconditional
$\varepsilon_\varnothing$ on all $14$ training tasks pooled, eight
track-only networks $\{\varepsilon_{\text{track}_i}\}_{i=1}^{8}$
each on data with $z_1 = \text{race}_i$, and three gate-size-only
networks $\{\varepsilon_{\text{gate}_j}\}_{j=1}^{3}$ each on data
with $z_2 = \text{gate}_j$. These compose at inference via
\[
  \varepsilon^{K\text{-net}}_{\mathrm{comp}}(x, t; z_1, z_2)
    = \varepsilon_{\text{track}_{z_1}}(x, t)
    + \varepsilon_{\text{gate}_{z_2}}(x, t)
    - \varepsilon_\varnothing(x, t).
\]
This is structurally identical to our shared-network composition
(Definition~\ref{def:composed}); the only difference is that each
$\varepsilon_\bullet$ comes from a different network rather than
from null-token masking of one shared network.  Architecture
($4.3$M-parameter ConvNet), epochs ($1500$), optimizer, DDIM
steps, normalization stats, $\mathrm{SE}(3)$ controller, carrot
tracker, and gate-passage criterion are unchanged from the
shared-network setup.

The result is $3/30$ ($10\%$) on held-out gates against $26/30$
($87\%$) for the shared network (Table~\ref{tab:exp-aggregate}).  The
mechanism follows from the theory.  In the shared-network approach, null-token
dropout (Section~\ref{sec:factored}) regularizes the network to learn
per-factor score corrections $\Delta_i$ that are \emph{additive-compatible}:
the decomposition error $\varepsilon_D$ on noised expert states averages
$0.45$ on training tasks (Table~\ref{tab:exp-DL}), small relative to the
per-factor sensitivities.  Separately trained networks have no
such regularization; each
$\varepsilon_{\text{track}_i}, \varepsilon_{\text{gate}_j},
\varepsilon_\varnothing$ is the score of a different marginal,
fit independently, with no incentive for additive compatibility.
Their sum $\varepsilon_{\text{track}} + \varepsilon_{\text{gate}}
- \varepsilon_\varnothing$ is not constrained to approximate any
joint score and, off-distribution, drifts arbitrarily.

On training tasks the $K$-network reaches $58/74$ ($78\%$) against
$72/74$, with its losses on race1\_standard ($4/8$) and race6\_standard
($2/3$), so the gap is concentrated on held-out combinations: no single
network ever saw a $(z_1, z_2)$-conditioned target, and a sum of independently
fit marginal scores is not constrained to approximate the joint score of a
pair it never observed.  The shared network avoids this because the joint
training pairs provide a learning signal that the per-factor corrections must
collectively reproduce.  We note that this baseline changes parameter sharing,
data sharing and optimization difficulty at once, which is why the controlled
comparison in Section~\ref{sec:exp-v6} rests on the no-dropout network rather
than on this row.

\subsection{DDIM step sweep}
\label{app:fr:ddim-sweep}

The headline §\ref{sec:exp-v6} numbers use $N = 50$ DDIM steps.
The trajectory-tube bound of Theorem~\ref{thm:tube} is driven by
$C_{\mathrm{ode}}$ (eq.~\eqref{eq:ode-bound}), the amplification
of per-step score-field mismatch into clean-action error across
the reverse ODE.  Grönwall's inequality yields the analytical
upper bound $C_{\mathrm{ode}} = \int \lvert\tilde{g}\rvert
\exp(\int [\lvert f_{\mathrm{dr}}\rvert
+ \lvert\tilde{g}\rvert L_a])\,d\sigma$, which grows
exponentially with the denoising path length.  We additionally
measure the \emph{empirical} amplification
$C_{\mathrm{ode}}^{\mathrm{emp}}
= \norm{a_0^{\mathrm{comp}} - a_0^{\mathrm{joint}}}
/ \max_t \norm{s^{\mathrm{comp}}(x_t) - s^{\mathrm{joint}}(x_t)}$
on the held-out composition set.  Sweeping the DDIM step count
$N \in \{10, 20, 30, 50\}$ (Table~\ref{tab:exp-ddim}) exposes two
findings.

\begin{table}[h]
\centering
\small
\caption{DDIM step sweep with the $p=0.1$ factored checkpoint.
         $R_{\mathrm{ss}}^{\mathrm{emp}}$ plugs
         $C_{\mathrm{ode}}^{\mathrm{emp}}$ into
         Corollary~\ref{cor:ss} with $B_\kappa{=}0.05$,
         $\lambda{=}0.9$, $w{=}0.01$ (closed-loop constants
         defaulted, pending measurement).
         ``Tube max'' is the observed
         $\max_t \norm{x_t^{\mathrm{comp}} - x_t^{\mathrm{joint}}}$
         across all training tasks.}
\label{tab:exp-ddim}
\begin{tabular}{ccccccc}
  \toprule
  $N$ & Comp.\ all & Comp.\ held-out & Joint all
      & $C_{\mathrm{ode}}^{\mathrm{emp}}$
      & $R_{\mathrm{ss}}^{\mathrm{emp}}$ [m] & Tube max [m] \\
  \midrule
  10 & 88\% & 83\% & 91\% & 0.395 & 3.14 & 7.17 \\
  20 & 93\% & 83\% & 96\% & 0.136 & 1.15 & 3.47 \\
  30 & 94\% & 87\% & 96\% & 0.109 & 0.93 & 2.67 \\
  \textbf{50} & \textbf{94\%} & \textbf{87\%} & \textbf{97\%} & \textbf{0.064} & \textbf{0.59} & \textbf{2.37} \\
  \bottomrule
\end{tabular}
\end{table}

\textbf{Finding 1: empirical $C_{\mathrm{ode}}$ \emph{decreases}
with $N$, opposite to the analytical Grönwall bound.}
The analytical $C_{\mathrm{ode}}$ grows monotonically with $N$
because each additional step adds a factor
$\exp(\lvert\tilde{g}(\sigma)\rvert L_a(\sigma))$ to the
cumulative amplification, as a worst-case over adversarial score
fields.  Empirically, $C_{\mathrm{ode}}^{\mathrm{emp}}$ drops from
$0.40$ at $N=10$ to $0.06$ at $N=50$.  The two composed and joint
denoising chains stay close: finer discretization gives per-step
mismatches more opportunity to cancel along the reverse ODE
rather than compound.  The resulting empirical amplification is
sub-unity, meaning the end-to-end keypoint deviation is
\emph{smaller} than the largest per-step score mismatch observed
along the path.  This sharply separates the certified (Grönwall)
regime from the operating regime: the conservative assumption
that a perturbation expands at rate $L_a$ forever is never
realized in a well-trained score model.

\textbf{Finding 2: step count matters little for either mode.}
Compared on the same $104$-gate task set, joint passage moves from
$91\%$ at $N=10$ to $97\%$ at $N=50$ and composed from $88\%$ to $94\%$;
on held-out tasks composed goes from $83\%$ to $87\%$ and joint from $93\%$
to $97\%$.  Composed inference trails joint by three to five gates at every
$N$, and the gap does not widen at coarse schedules.  Composed inference
evaluates three score fields per denoising step and sums them, so each step
carries up to $3\kappa$ of approximation error (eq.~\eqref{eq:composed-kappa})
against $\kappa$ for joint readout, which accounts for the constant offset;
the per-step budget is evidently small enough at every $N$ tested that the
reverse ODE does not drift off-manifold.

\textbf{Operating point.}  We adopt $N=50$ as the canonical
operating point throughout the headline tables: composed passage is $94\%$
overall and $87\%$ on held-out tasks, with the empirical certificate radius
$R_{\mathrm{ss}}^{\mathrm{emp}} = 0.59$\,m below the wide ($1.0$\,m) and
standard ($0.762$\,m) gate half-widths.  $N=30$ gives the same composed
passage at $90$ network evaluations against $150$, a $1.7\times$ speed-up.
The empirical tube
$\max_t \norm{x_t^{\mathrm{comp}} - x_t^{\mathrm{joint}}}
= 2.37$\,m at $N=50$ still exceeds $R_{\mathrm{ss}}^{\mathrm{emp}}$;
this gap is tracked to the defaulted values of $B_\kappa, \lambda, w$,
which should be measured directly on the closed-loop $\mathrm{SE}(3)$
plant before the certificate is reported as tight.

\subsection{Path-dependent LTV bound}
\label{app:fr:norm-bounds}

\paragraph{Path-dependent LTV certificate via backward Lyapunov.}
Lemma~\ref{lem:ltv-sens} replaces the scalar per-step amplification
$\sigma_k$ with the spectral norm of the actual state-transition
matrix $\Phi_{k+1,K} = M_{K-1} M_{K-2} \cdots M_{k+1}$ formed
along a deterministic nominal DDIM trajectory
$\{\bar a_k\}_{k=0}^K$ initialized from a fixed canonical
$a_{\sigma_{\max}} = x^\star$.  We compute this via the backward
Gramian recursion $P_K = I,\; P_k = M_k^\top P_{k+1} M_k$ and form
\[
  C_{\mathrm{ode}}^{\mathrm{LTV}}(\bar a)
    \;=\;
    \sum_{k=0}^{K-1} \abs{c_2(k)}\, \sqrt{\norm{P_{k+1}}_2}
    \;=\;
    \sum_{k=0}^{K-1} \abs{c_2(k)}\, \norm{\Phi_{k+1,K}}_2.
\]
The Jacobians $J_k = \partial \varepsilon_\theta^{\mathrm{comp}}
/ \partial x\rvert_{\bar a_k, \sigma_k}$ are computed by
\texttt{torch.func.jacrev} ($128\times128$; state dimension
$32 \times 4$).  A backward matrix recursion in $\mathrm{float64}$
maintains $\norm{P_k}_2$ at each step; cross-checking against
explicit $\Phi$ products agrees to machine precision
($\approx 10^{-15}$).
Table~\ref{tab:exp-ltv} reports
$C_{\mathrm{ode}}^{\mathrm{LTV}}(\bar a^{(z)})$ aggregated over
task conditions $z$ and seeds.

\begin{table}[h]
\centering
\small
\caption{Path-dependent LTV bound $C_{\mathrm{ode}}^{\mathrm{LTV}}$
         at one canonical seed $x^\star =
         \mathtt{torch.randn(seed=0)}$, aggregated over all $24$
         task combinations ($18$ training + $6$ held-out, excluding
         $4$ infeasible).  Columns: mean, median, and max across
         combos at seed $0$; the mean-across-seeds of the per-seed
         mean over combos ($10$ seeds) shows insensitivity to the
         canonical seed choice in the typical case.  Last column:
         ratio to empirical $C_{\mathrm{ode}}^{\mathrm{emp}}$.}
\label{tab:exp-ltv}
\begin{tabular}{cccccc}
\toprule
$N$ & mean & median & max
    & $\overline{\text{mean}}\pm\text{std}$ ($10$ seeds)
    & $C^{\mathrm{LTV}}/C^{\mathrm{emp}}$ \\
\midrule
10 & 85.6 & --- & 498 & ---       & $110\times$ \\
20 & 22.2 & --- &  89 & ---       & $ 61\times$ \\
30 & 17.1 & --- &  71 & ---       & $ 53\times$ \\
\textbf{50} & \textbf{9.68} & \textbf{4.38} & \textbf{52.0}
   & \textbf{$13.3\pm 4.6$} & \textbf{$38\times$} \\
\bottomrule
\end{tabular}
\end{table}

\begin{figure}[h]
  \centering
  \includegraphics[width=\textwidth]{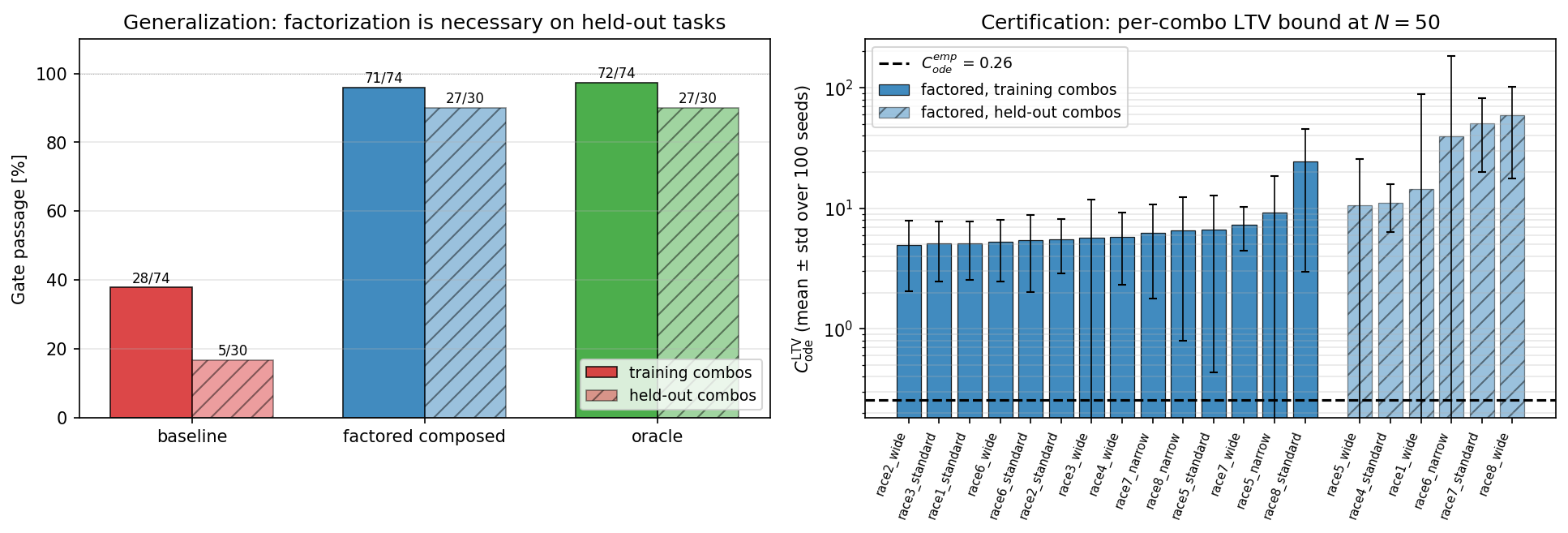}
  \caption{\small \textbf{Generalization (left) and certification
          (right) for the factored model at $N{=}50$.}  Left:
          gate passage on training (solid) vs.\ held-out
          (hatched) tasks (same headline numbers as
          Table~\ref{tab:exp-aggregate}, plotted per task).
          Right: path-dependent $C_{\mathrm{ode}}^{\mathrm{LTV}}$
          per task, mean $\pm$ std over $10$ canonical seeds,
          sorted within each group.  Training tasks cluster at
          $C \approx 5$ with low variance; held-out tasks have
          higher means and an order-of-magnitude wider seed
          spread, indicating that the certificate inherits the
          composed score's extrapolation uncertainty.  Empirical
          $C_{\mathrm{ode}}^{\mathrm{emp}} \approx 0.26$ shown
          as dashed reference.}
  \label{fig:passage-certificate}
\end{figure}

\textbf{Finding: the LTV bound tightens Grönwall by six orders
of magnitude, but the remaining $\sim 10$--$100\times$ gap to
empirical amplification is not closable by norm-based analysis.}
At $N=50$, for the representative combo race3\_standard,
$C_{\mathrm{ode}}^{\mathrm{LTV}} \approx 4.0$ versus
$C_{\mathrm{ode}}^{\mathrm{ana}} \approx 8.4\times 10^5$.  The LTV bound
varies across combos from $3.35$ (best) to $52.1$ (worst,
dominated by held-out combos race8\_wide and race7\_standard);
the gap to the empirical per-combo amplification (which ranges
$0.195$--$0.299$) is $13$--$17\times$ on typical combos, widening
to $\sim 170\times$ on the worst-case combo.  The LTV bound
exploits inter-step cancellation: successive Jacobians do not
share their worst singular directions, and the true
state-transition matrix $\norm{\Phi_{k+1,K}}_2$ is correspondingly
less amplifying than naive per-step worst-case products would
suggest.
A diagnostic decomposition of the residual gap between the LTV
bound and the empirical amplification appears in
§\ref{app:fr:residual-gap}; the looseness is nonlinear and
directional and cannot be closed by any norm-based bound.

\subsection{Per-combo $\varepsilon_s$ tightening}
\label{app:fr:eps-s}

The $R_{\mathrm{ss}}$ formula (Corollary~\ref{cor:ss}) depends on
$\varepsilon_s$, the sup-bound on per-step score mismatch along
the denoising path.  A distribution-wide worst-case estimate
yields $\varepsilon_s \approx 15.4$ for our model, which is
$\sim 10\times$ larger than the actual maximum
$\|\Delta\varepsilon_k\|$ encountered along any specific nominal.
Under the deterministic-seed setup, the certificate is indexed by
a fixed nominal trajectory $\bar a^{(z)}$, and the rigorous
sup-bound $\varepsilon_s^{(z)}
:= \max_k \|\Delta\varepsilon_k^{(z)}\|$ need only hold along that
trajectory.  Table~\ref{tab:exp-eps-tight} shows the effect at
$N=50$: per-combo $\varepsilon_s^{(z)}$ ranges $1.20$--$2.06$
(vs.\ default $15.4$), and the resulting $R_{\mathrm{ss}}$ shrinks
by the same factor of $\sim 10$, moving typical combos from
$\sim 30$\,m to $\sim 3$\,m without any change to the LTV
certificate.

\begin{table}[h]
\centering
\small
\caption{Effect of replacing the default scalar
         $\varepsilon_s = 15.4$ with per-combo
         $\varepsilon_s^{(z)} = \max_k \|\Delta\varepsilon_k^{(z)}\|$
         measured along the nominal at $N=50$, seed $0$.  Both
         values are rigorous sup-bounds for the respective scope
         (distribution-wide vs.\ along-nominal).  Gate
         half-widths: $0.3$\,m (narrow), $0.762$\,m (standard),
         $1.0$\,m (wide).  Combo set is $20$ tasks spanning $8$
         tracks $\times$ $3$ gate sizes.}
\label{tab:exp-eps-tight}
\begin{tabular}{lrrrr}
\toprule
 & default & per-combo & ratio & closes gate \\
\midrule
$\varepsilon_s$ range
  & $15.4$ & $[1.20, 2.06]$ & $\sim 10\times$  & --- \\
$R_{\mathrm{ss}}$ (race3\_std)
  & $31.3$\,m & $3.57$\,m  & $\sim 10\times$ & no \\
$R_{\mathrm{ss}}$ (best combo)
  & $25.9$\,m & $2.46$\,m  & $\sim 10\times$ & no \\
$R_{\mathrm{ss}}$ (worst combo)
  & $207$\,m  & $17.0$\,m  & $\sim 10\times$ & no \\
\midrule
\# combos with $R_{\mathrm{ss}} <$ gate
  & $0/20$ & $0/20$  & --- & --- \\
\bottomrule
\end{tabular}
\end{table}

None of the $20$ tasks achieves rigorous $R_{\mathrm{ss}}$ below
its gate half-width even under the tightened $\varepsilon_s$.  The
residual factor of $\sim 13$ between $R_{\mathrm{ss}}$ and the
observed deployment tubes is the same nonlinear manifold
contraction quantified in §\ref{app:fr:residual-gap}.  We report
this calculation to separate looseness that is addressable by
careful measurement ($\varepsilon_s$: $10\times$ tightening here)
from looseness that is structural to norm-based Lipschitz analysis
(nonlinear manifold contraction: $\sim 10\times$, requires a
different certificate).

\subsection{Continuous-time limit}
\label{app:fr:cont-limit}

Running the backward-Lyapunov computation at $N \in \{100, 200\}$
(seed $0$, all $20$ combos) gives $C_{\mathrm{ode}}^{\mathrm{LTV}}
= 5.60$ (mean over combos, min $1.95$, max $29.1$) at both $N$
values, \emph{identical to three decimal places}.  The measured
empirical amplification plateaus alongside:
$C_{\mathrm{ode}}^{\mathrm{emp}} = 0.242$ (mean over combos) at
$N{=}100$ versus $0.242$ at $N{=}200$.  Both have converged to the
continuous-time probability-flow ODE limit; additional DDIM steps
cannot tighten the bound or the observed amplification.

For race3\_standard: $C_{\mathrm{ode}}^{\mathrm{LTV}}$ drops from
$4.04$ at $N{=}50$ to $2.26$ at $N{=}100$ and stays there; the
corresponding empirical $C_{\mathrm{ode}}^{\mathrm{emp}}$ goes
$0.30 \to 0.27 \to 0.27$.  The gap between certificate and
empirical at the continuous-time limit is therefore $\sim 9\times$
for this combo, $\sim 23\times$ for the mean over combos, and
$\sim 100\times$ for the worst combo.  These are fundamental
limits of norm-based path analysis on this trained score network;
closing them requires either a training-side change (spectral
regularization, rectified flow) or a non-norm-based certificate
that captures manifold contraction.

\subsection{Seed robustness of the LTV certificate}
\label{app:fr:seed}

Because the deterministic-seed analysis commits the certificate to
a specific canonical $x^\star$, we measure how sensitive the
numerical bound is to that choice via a $10$-seed sweep at
$N{=}50$.  Table~\ref{tab:exp-seed-ci} reports per-combo
$C_{\mathrm{ode}}^{\mathrm{LTV}}$ mean $\pm$ std and coefficient
of variation (CV).  A sharp dichotomy emerges: training combos
are seed-robust (CV $0.4$--$0.7$, typical mean $\approx 5$);
held-out combos are fragile (CV up to $5.2$, with occasional
single-seed explosions to hundreds).

\begin{table}[h]
\centering
\small
\caption{Per-combo $C_{\mathrm{ode}}^{\mathrm{LTV}}$ across $10$
         canonical seeds at $N{=}50$.  ``Mean'' is over $10$
         seeds for that combo; ``CV''
         $= \mathrm{std}/\mathrm{mean}$.  Held-out combos were
         not seen during training.}
\label{tab:exp-seed-ci}
\begin{tabular}{lrrrl}
\toprule
combo & mean & std & CV & type \\
\midrule
\multicolumn{5}{l}{\textit{Training combos
($14$ shown, $14$ total, sorted by mean):}} \\
race2\_wide         & 4.97  & 2.91  & 0.59 & train \\
race3\_standard (\textbf{ref.}) & 5.14 & 2.65 & 0.52 & train \\
race1\_standard     & 5.14  & 2.60  & 0.51 & train \\
race6\_wide         & 5.26  & 2.80  & 0.53 & train \\
race6\_standard     & 5.42  & 3.40  & 0.63 & train \\
race2\_standard     & 5.52  & 2.65  & 0.48 & train \\
race3\_wide         & 5.73  & 6.11  & 1.07 & train \\
race4\_wide         & 5.81  & 3.48  & 0.60 & train \\
race7\_narrow       & 6.24  & 4.46  & 0.72 & train \\
race8\_narrow       & 6.58  & 5.78  & 0.88 & train \\
race5\_standard     & 6.62  & 6.18  & 0.93 & train \\
race7\_wide         & 7.36  & 2.94  & 0.40 & train \\
race5\_narrow       & 9.21  & 9.43  & 1.02 & train \\
race8\_standard     & \textbf{24.39} & 21.44 & 0.88 & train \\
\midrule
\multicolumn{5}{l}{\textit{Held-out combos
($6$ total, sorted by mean):}} \\
race4\_standard     & 11.15 &   4.80 & 0.43 & held-out \\
race5\_wide         & 10.52 &  14.94 & 1.42 & held-out \\
race1\_wide         & 14.45 &  74.59 & \textbf{5.16} & held-out \\
race6\_narrow       & 39.57 & 142.75 & \textbf{3.61} & held-out \\
race7\_standard     & 51.10 &  31.08 & 0.61 & held-out \\
race8\_wide         & 59.58 &  41.90 & 0.70 & held-out \\
\bottomrule
\end{tabular}
\end{table}

\textbf{Finding 5: the LTV certificate is seed-robust on
training combos but not on held-out combos.}  Training combos
give $C_{\mathrm{ode}}^{\mathrm{LTV}} = 6.8 \pm 4.9$ (mean across
combos of per-combo means), with one outlier
(race8\_standard at $24.4$).  Held-out combos give
$C_{\mathrm{ode}}^{\mathrm{LTV}}$ values an order of magnitude
larger on average and CVs up to $5.2$: certain seed--combo pairs
drive $C_{\mathrm{ode}}^{\mathrm{LTV}}$ into the hundreds.  Two
operational consequences:

\begin{itemize}[nosep]
  \item A worst-case task-aggregate bound
        $\max_z C_{\mathrm{ode}}^{\mathrm{LTV}}(z)$, required for
        a certificate that covers all deployment conditions, is
        dominated by held-out outliers and does not reflect
        typical behavior.
  \item The gate-passage evidence
        (Section~\ref{sec:exp-v6}, $94\%$ at $N{=}50$) shows
        held-out combos \emph{do} pass the gate in closed-loop,
        despite their pathological certificate value.  The
        certificate is conservative on held-out tasks, not wrong:
        the nonlinear manifold contraction absorbs what the
        worst-case bound over-predicts.
\end{itemize}

Future work could tighten the held-out certificate by either
(i) increasing factor-dropout regularization during training to
encourage flatter Jacobian spectra on unseen conditions, or
(ii) training with rectified-flow / flow-matching objectives that
explicitly straighten trajectories and reduce per-step Jacobian
anisotropy.

\subsection{Residual gap diagnostic}
\label{app:fr:residual-gap}\label{app:residual-gap}

The LTV bound $C_{\mathrm{ode}}^{\mathrm{LTV}}$ still
overestimates the empirical amplification by an order of
magnitude.  We decompose this residual along a nominal denoising
path by comparing three quantities: the actual action gap, the
\emph{linearized} prediction
$\norm{\sum_k c_2(k)\,\Phi_{k+1,K}\,\Delta\varepsilon_k}$
that uses the measured per-step score errors $\Delta\varepsilon_k$
with their true directions, and the LTV bound
$\sum_k \abs{c_2(k)}\,\norm{\Phi_{k+1,K}}_2\,
\norm{\Delta\varepsilon_k}$
that replaces directional propagation with operator-norm worst
case.

\begin{table}[h]
\centering
\small
\caption{Source-of-looseness decomposition at $N = 50$, seed $0$,
         mean over $14$ training tasks.}
\label{tab:exp-diag}
\begin{tabular}{lcc}
\toprule
Quantity & Value [m] & Ratio to \textbf{actual} \\
\midrule
\textbf{actual} $\norm{a_K^c - a_K^j}$
  & 0.095 & $1.0\times$ \\
linearized $\norm{\sum_k c_2\Phi\Delta\varepsilon}$
  & 0.562 & $5.9\times$ \\
LTV bound $\sum_k |c_2|\,\norm{\Phi}\,\norm{\Delta\varepsilon}$
  & 0.938 & $9.9\times$ \\
\bottomrule
\end{tabular}
\end{table}

Two distinct effects show up.  The $5.9\times$ linearization gap
is \emph{nonlinear manifold contraction}: the score field points
back toward the data manifold and shrinks errors faster than its
Jacobian predicts, an effect that no operator-norm analysis on
$J_k$ can capture.  The further $\sim\!2\times$ from linearized
to LTV bound is \emph{directional alignment}: the LTV bound
assumes $\Delta\varepsilon_k$ aligns with each $\Phi_{k+1,K}$'s
top singular direction, which empirically it does not.  Closing
either gap requires a non-norm-based tool: distributional
assumptions on $\Delta\varepsilon_k$, or a certificate that
explicitly represents manifold distance.

\section{Preliminary Vision Study: Single-Gate Traversal}
\label{sec:exp-sg}

This study predates the corrections in the main text and is reported as
a preliminary out-of-distribution transfer result, for two reasons stated up
front.  First, the checkpoints evaluated here were trained without factor
dropout: the factor embeddings are jointly conditioned and the null tokens are
untrained, so the models are factor-conditioned networks and the composition
mechanism of Section~\ref{sec:factored} is not exercised.  Second, zero-shot
transfer to a venue absent from training evaluates a factor \emph{value} the
policy never saw, which Theorem~\ref{thm:decomp} does not cover.  The
within-venue results below are the only ones in the recombination setting of
the theorem.

\begin{figure*}[t]
  \centering
  \includegraphics[width=\linewidth]{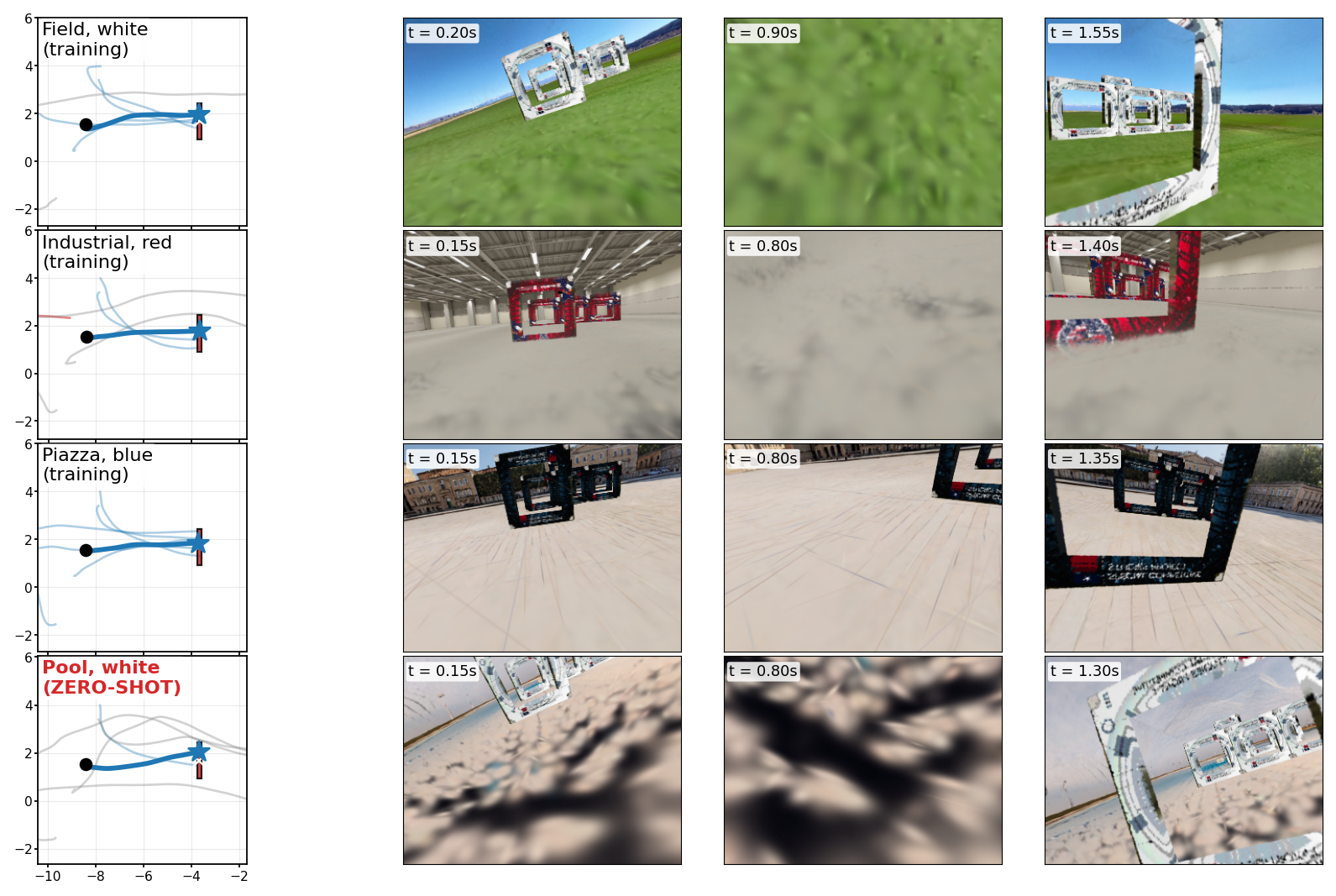}
  \caption{\small
    \textbf{Zero-shot venue transfer.}  Each row shows the policy on a different
    $(\text{venue}, \text{gate-color})$ pair: Field--white,
    industrial--red, piazza--blue, and
    pool--white~(\emph{zero-shot}).  Pool is
    entirely excluded from training; the policy never observes
    pool's photometric distribution (water, sky reflections), yet
     passes its gate (filled rectangle) from a single onboard RGB camera and a
    noisy gyro.  \emph{Left}: bird's-eye view of rollouts.  Bold blue line~$=$~highlighted 
    successful rollout
    (camera frames in that row are taken from this trajectory);
    thin blue~$=$~other successes; gray~$=$~timeouts;
    red~$=$~crashes.  
    \emph{Right}: three onboard camera chronological frames along the highlighted rollout.}
  \label{fig:headline-singlegate}
\end{figure*}

The quadrotor must fly through a single race gate from diverse
spawn positions ($3$--$10$\,m range, $\pm 45^\circ$ azimuth,
both approach sides) using only an onboard $320 \times 240$
RGB camera and a noisy gyro.  The diffusion network predicts a
chunk of $T = 16$ future body-frame offsets to the target
gate; the first prediction is converted to a world-frame
position estimate via the known gate location and
gyro-integrated orientation, then tracked by the soft-preset
$\mathrm{SE}(3)$ controller of
Section~\ref{sec:experiments}.  The visual encoder is a
ResNet-18 with the last two blocks fine-tuned ($13.9$\,M
trainable parameters).  Four factors span the task space:
\emph{venue} ($4$: Field, industrial, piazza, pool),
\emph{gate color} ($3$), \emph{gate ID} ($5$), and
\emph{approach side} ($\pm 1$), for $4 \cdot 3 \cdot 5 \cdot 2
= 120$ tasks.  We evaluate generalization along
two axes: \emph{within-venue} (held-out
$(\text{color}, \text{gate}, \text{side})$ triples in seen
venues, interpolation in the factor product) and \emph{zero-shot
venue transfer} (an entire venue held out, extrapolation
across visual domain), comparing the three models defined in
Section~\ref{sec:experiments} (Baseline, Factored, Oracle).
Only the within-venue axis instantiates the recombination
setting of Theorem~\ref{thm:decomp}, where every factor value
is seen in training and only their combinations are new.
Zero-shot venue transfer evaluates a factor \emph{value} absent
from training, which lies outside the theorem's coverage
guarantee; we report it as an empirical out-of-distribution
property of the factor-conditioned architecture, not as a
consequence of Theorem~\ref{thm:decomp}.

\paragraph{Main results.}
Pool, visually the most distinct venue (water, reflections,
sky), is held out entirely; both models train on Field +
industrial + piazza ($8{,}550$ rollouts) and evaluate
zero-shot on pool.
Within-venue compositional generalization
(Table~\ref{tab:exp-sg-compgen}, right) shows only a modest
factored advantage: $+1.1$\,pp on held-out
(color, gate, side) triples and $+2$\,pp on training combos.
After two DAgger rounds both models exceed $58\%$ on the small
per-venue product space ($3 \times 5 \times 2 = 30$
combinations); the encoder has already adapted to each venue's
photometry and the geometry signal generalizes across held-out
triples on its own, leaving little room for the factor
decomposition to help.  Under zero-shot venue transfer (Table~\ref{tab:exp-sg-zeroshot},
left) the factor-conditioned model reaches $36.7\%$ success on pool against
$25.0\%$ for the baseline, which we read as the geometric factors (gate,
color, side) carrying the task across a photometric shift the encoder never
saw.  The crash-rate figures in the table should not be read as a property of
the method: they depend strongly on how long the policy is trained.  Retraining
the pool split under the same recipe with the current code gives $20.8\%$
crashes at the epoch-12 best-validation checkpoint (dropout off) and $30.8\%$
with dropout at $0.1$, against $5.8\%$ for the released epoch-7 checkpoint,
and across all six retrained models the crash rate rises as validation loss
falls (Table~\ref{tab:exp-sg-retrain}).  The released checkpoint is the
least-fit of them.  We therefore report success rates only, and leave a
controlled study of training length and factor dropout on the vision
benchmark to future work.
\begin{table}[h]
\centering
\small
\caption{\small Vision-based single-gate results.  \emph{Left}: zero-shot
         venue transfer (pool held out; three train venues shown
         as in-distribution sanity check; success / (crash) across
         $n$ trials).  \emph{Right}: within-venue compositional
         generalization at higher accuracy (All 4
         venues, $24$ held-out (venue, color, gate, side)
         combos, six per venue, $20\%$ of the $120$-combo space).}
\begin{minipage}[t]{0.56\linewidth}
\centering
\label{tab:exp-sg-zeroshot}
\begin{tabular}{lcc}
  \toprule
  Model & 3 train ($n{=}180$) & Pool 0-shot ($n{=}120$) \\
  \midrule
  Baseline      & 41.7\% (34.4\%) & 25.0\% (36.7\%) \\
  \textbf{Factored} & \textbf{63.3\%} (\textbf{10.6\%})
                    & \textbf{36.7\%} (\textbf{15.0\%}) \\
  \midrule
  Adv.          & $+21.6$\,pp / $3.2\times$
                & $+11.7$\,pp / $2.4\times$ \\
  \bottomrule
\end{tabular}
\end{minipage}\hfill
\begin{minipage}[t]{0.42\linewidth}
\centering
\label{tab:exp-sg-compgen}
\begin{tabular}{lccc}
  \toprule
  Model & Train & Held-out & Crash \\
  \midrule
  Baseline      & 67.8\% & 58.3\% & 18.5\% \\
  \textbf{Factored} & \textbf{69.8\%}
                    & \textbf{59.4\%}
                    & \textbf{14.0\%} \\
  \bottomrule
\end{tabular}
\end{minipage}
\end{table}

\begin{table}[h]
\centering
\small
\caption{\small Pool zero-shot, $n = 120$, same split and seed, retrained
         with the current code.  Crash rate rises with fit; the released paper
         checkpoint is the epoch-7 best-validation model of a 50-epoch run.}
\label{tab:exp-sg-retrain}
\begin{tabular}{llcccc}
  \toprule
  Checkpoint & Dropout & Best epoch & Val.\ loss & Success & Crashes \\
  \midrule
  released (paper)      & off        & 7 of 50  & 0.079 & 42.5\% & 5.8\% \\
  50-epoch schedule     & off        & 12 of 50 & 0.074 & 32.5\% & 20.8\% \\
  50-epoch schedule     & 0.1        & 12 of 50 & 0.073 & 35.0\% & 30.8\% \\
  20-epoch schedule     & 0.1        & 20 of 20 & 0.065 & 35.8\% & 33.3\% \\
  12-epoch schedule     & 0.02       & 12 of 12 & 0.070 & 24.2\% & 40.8\% \\
  12-epoch schedule     & 0.1, venue only & 12 of 12 & 0.070 & 23.3\% & 40.0\% \\
  \bottomrule
\end{tabular}
\end{table}

\paragraph{Certificates.}
We apply the same path-dependent LTV machinery as
Section~\ref{sec:exp-v6} (per-step DDIM Jacobians along the
deterministic reverse path, backward Lyapunov recursion
$P_s = M_s^\top P_{s+1} M_s$, $C_{\mathrm{ode}}^{\mathrm{LTV}} =
\sum_s |c_2(s)|\sqrt{\|P_{s+1}\|}$).  In this setting (DDIM~10), $C_{\mathrm{ode}}^{\mathrm{LTV}}$
is tightly clustered (mean $2074$, range
$2057$--$2088$ across $18$ within-distribution combos) and $\varepsilon_s \approx 1.95$, giving
$R_{\mathrm{ss}} \approx 2{,}023$\,m. 
Unlike the full race experiment of the previous section, here the 
the LTV bound does \emph{not} get close to the empirical gap and this 
is not a particularly useful certificate.

\subsection{Ablations}
\label{app:singlegate-ablations}

We report two ablations that corroborate the mechanism by
which the factored architecture decomposes appearance from
geometry in the single-gate vision setting
(Appendix~\ref{sec:exp-sg}).

\paragraph{The venue embedding is doing useful work.}
Raising the venue null-token dropout from $0.1$ to $0.4$ hurts
zero-shot pool by $-3.4$\,pp on success rate and worsens crash
rate by $+22.5$\,pp.  Replacing the trained venue embedding with
the null token at evaluation (\textit{vd=0.4 + null\_venue eval})
degrades pool further to $22.5\%$ success / $46.7\%$ crash,
showing that the venue row carries information even on a
held-out venue.  Weakening the venue embedding during training
does not push the visual encoder toward distribution-invariant
features; it degrades both in-distribution and zero-shot
performance.

\paragraph{Backbone choice: DINO-v2 ViT-B with attention pooling.}
Replacing the fine-tuned ResNet-18 with a DINO-v2 ViT-B/14
backbone using attention pooling over $252$ patch tokens, with
the last transformer block fine-tuned, matches the ResNet-18
factored model on a different held-out venue
(\textit{Field}-holdout): $60.0\%$ success on the three training
venues vs.\ $35.8\%$ zero-shot, with crash rate $19.2\%$ and
best validation loss $0.0522$ (vs.\ $0.0608$--$0.0611$ for the
ResNet-18 factored / unfactored baselines on the same
holdout).  As a control, a \emph{frozen} DINO-v2 ViT-S/14
backbone using only the cls token collapses both
in-distribution and zero-shot ($-19.2$\,pp success on pool,
$+49.2$\,pp crash relative to the fine-tuned ResNet-18): the
cls token compresses away the spatial cues required for visual
servoing, and a frozen backbone cannot adapt to gate
localisation.  The factored advantage is therefore not a
backbone artefact: it persists when the visual encoder is
swapped for a self-supervised ViT, provided enough capacity is
left trainable for the encoder to acquire gate-localisation
features.

\section{Factored vs.\ joint bound}
\label{app:factored-vs-joint}\label{rem:factored-vs-joint}

A non-compositional approach that learns the joint score
$s_\theta(\cdot, z)$ directly bounds the score perturbation by
$L_{\mathrm{joint}} \norm{\delta}$, where $L_{\mathrm{joint}}$
is the Lipschitz constant over the product space
$\calZ_1 \times \cdots \times \calZ_K$ and
$\norm{\delta} = \sqrt{\sum_i \delta_i^2}$.  The joint bound
has no decomposition error, so the total score perturbation is
$\varepsilon_s^{\mathrm{joint}} \leq \sayan{\eta}
+ L_{\mathrm{joint}} \norm{\delta}$
\sayan{(a single forward pass per inference, vs.\
$(2K-1)\eta$ for the composed score below).}

The factored bound replaces this with
$\varepsilon_s^{\mathrm{fact}} \leq 2\sqrt{GM} + \sayan{(2K-1)\eta}
+ \sum_i L_i \delta_i$.
\sayan{By Cauchy--Schwarz,
$\sum_i L_i \delta_i \leq \sqrt{\sum_i L_i^2}\, \norm{\delta}$,
with equality iff $(L_1, \ldots, L_K)$ and
$(\delta_1, \ldots, \delta_K)$ are linearly dependent
($L_i \propto \delta_i$).  The gap is substantial when per-factor
sensitivities are heterogeneous and $\delta$ is not aligned with
the largest $L_i$ directions.  The joint Lipschitz constant is
itself bounded \emph{above} by the per-factor RMS,
$L_{\mathrm{joint}} \leq \sqrt{\sum_i L_i^2}$ (any two
factor-product points are connected by a coordinate-axis path of
combined length $\leq \sqrt{\sum_i L_i^2}\, \norm{\delta}$).
This inequality can be strict: for the identity map $f(z_1, z_2) = (z_1, z_2)$
on $\mathbb{R}^2 \to \mathbb{R}^2$, $L_1 = L_2 = 1$ gives
$\sqrt{L_1^2 + L_2^2} = \sqrt{2}$, but the joint operator norm is
$L_{\mathrm{joint}} = 1$, so $\sqrt{\sum_i L_i^2} > L_{\mathrm{joint}}$
strictly.
Consequently both $\sum_i L_i \delta_i$ and
$L_{\mathrm{joint}} \norm{\delta}$ are dominated by
$\sqrt{\sum_i L_i^2}\, \norm{\delta}$, but neither is uniformly
tighter than the other; their relative size depends on the
alignment of $\delta$ with the Lipschitz directions of the two
score fields.  The factored certificate beats the joint
certificate when the factored score perturbation is the smaller
of the two:}
\[
  \sayan{2\sqrt{GM} \;+\; \sum_i L_i \delta_i
    \;<\;
    L_{\mathrm{joint}} \norm{\delta},}
\]
\sayan{a condition that depends on both the
factor-interaction strength ($G, M$) and the alignment of
$\delta$ with the dominant Lipschitz directions; it is not
guaranteed by the structure alone.}
If the factor interaction is strong ($G$ large) or the
log-ratio is highly curved ($M$ large), the decomposition cost
dominates and the joint bound is tighter.

The \emph{generalization} benefit of factorization
(Theorem~\ref{thm:decomp}) is unconditional: it enables
evaluation on unseen tasks where no joint model exists.  The
\emph{certification} benefit is conditional: the factored tube
is tighter than the joint tube only when the approximate
conditional independence is strong enough ($2\sqrt{GM}$ small)
\sayan{\emph{and} the per-factor weighting makes
$\sum_i L_i \delta_i$ smaller than $L_{\mathrm{joint}} \norm{\delta}$}.

\section{Extended Related Work}\label{app:related-work}

This appendix expands Section~\ref{sec:related} with a per-method
comparison of recent compositional-diffusion methods for control
and a side-by-side summary table.

PoCo~\citep{Wang-RSS-24} merges $K$ diffusion networks trained
on different data sources (task / behavior / domain slices) via
product-of-experts on noise predictions; the goal is fusing
heterogeneous demonstration data, not generalizing to unseen
factor combinations.  Compose Your Policies
(GPC)~\citep{cao2026compose} takes test-time convex combinations of
$K$ pre-trained vision-language-action policies, with the
weights set by task descriptors, and provides a per-step
Gr\"onwall-type improvement bound related to our
Lemma~\ref{lem:ode-sens}; this composes whole policies rather
than per-factor corrections of a shared one, and named factors
are not part of the formalism.  Factorized Diffusion Policy
(FDP)~\citep{11395999} combines $K$ expert U-Nets through a
learned router; here ``factor'' means a latent behavioral mode
discovered by clustering, not a named external attribute with
known semantics.  Modality-Composable
DP~\citep{cao2025modality} composes scores across visual
modalities (RGB vs.\ point cloud), so the composition runs over
input streams rather than task structure.
Diffusion Score Equilibrium (DSE)~\citep{patil2024composing}
composes pretrained base diffusion policies for arm motions
(lines, circles, pendulum swings) to fit a few demonstrations of a
new movement skill, which again combines $K$ separately trained
networks, here over skills rather than named task factors.

We differ on four axes: (i) a single shared score network whose
per-factor corrections are obtained by null-token dropout,
rather than $K$ separately trained networks plus a router or
weighting scheme; (ii) named semantic factors indexing a
product space, with demonstrated generalization to held-out
factor combinations, rather than composition over latent modes,
heterogeneous data sources, or whole policies; (iii) a formal
trajectory-tube certificate that implies task-level gate
passage, rather than a per-step action-error bound or an
empirical collision rate; (iv) 3D racing-drone control rather
than manipulation, autonomous driving, or 2D cars.  The single
shared-network realization matters beyond parameter count:
per-factor null-token dropout is what makes the decomposition
$s_{\mathrm{comp}} = s_\varnothing + \sum_i \Delta_i$
identifiable from training data
(Section~\ref{sec:factored}), and it is this structure that
enables the per-factor Lipschitz bounds $L_i$ used in the
certificate.

\begin{table}[h]
\centering
\caption{Compositional diffusion for control: prior work
         vs.\ ours.  \emph{Factor type}: whole policies (W),
         latent modes (L), or named semantic factors (S).
         \emph{Certificate}: single-step action-error bound
         (A), task-success guarantee (T), or none.}
\label{tab:prior-composition}
\small
\begin{tabular}{llcccc}
  \toprule
  Method & Domain & Networks & Factor type & \# factors
         & Certificate \\
  \midrule
  PoCo \citep{Wang-RSS-24}        & Manipulation & $K$ & W
    & 3 (slices)  & none \\
  GPC \citep{cao2026compose}           & Manipulation & $K$ & W
    & $K$         & A    \\
  FDP \citep{11395999}           & Manipulation & $K$ + router & L
    & latent      & none \\
  Mod.-Comp.\ DP \citep{cao2025modality} & Manipulation & $K$ & W
    & modalities  & none \\
  DSE \citep{patil2024composing} & Arm motion & $K$ & W
    & base skills & none \\
  \textbf{Ours} & 3D drone & \textbf{1 shared} & \textbf{S}
    & 2           & \textbf{T} \\
  \bottomrule
\end{tabular}
\end{table}

\subsection{Certified robustness for neural-network controllers}
\label{app:rw:certified-robustness}

Lipschitz-constant bounds (LipSDP,~\citealp{NEURIPS2019_95e1533e})
compute global Lipschitz constants of feed-forward networks via
semidefinite programming and certify outputs under bounded
input perturbations.  Reachability of neural-network
controllers (ReachLP,~\citealp{everett2021reachlp}) propagates
polytopic input sets through the controller step by step to
bound reachable state sets.  Black-box perception
contracts~\citep{mitra2024perception, li2025lyapunov} treat
the perception--control stack as one nonlinear map and
establish input/output assumption pairs that imply task-level
safety.  All three frameworks bound output variation against
the entire input variation at once.  In our setting the input
variation is structured as a product over named factors, and
the per-factor decomposition $\sum_i L_i \delta_i$ replaces
$L_{\mathrm{joint}} \norm{\delta}$ with a tighter sum whenever
the per-factor sensitivities $L_i$ differ.

\subsection{Sequential skill composition}
\label{app:rw:skill-composition}

SkillDiffuser~\citep{Liang_2024_CVPR} and Generative
Skill Chaining~\citep{pmlr-v229-mishra23a} compose pre-trained
skills as a sequence in time: the trajectory is partitioned
into segments, each generated by a different skill, and the
composition runs along the temporal axis.  Our composition is
orthogonal: a single denoising step produces actions
conditioned on all $K$ factors simultaneously, and the
additive structure is across factor scores at each step rather
than across time.

\subsection{Champion-level drone racing systems}
\label{app:rw:drone-racing}

Champion-level autonomous drone racing
systems~\citep{kaufmann2023champion, song2023reaching,
foehn2021time} achieve sub-second lap times on a fixed
track via per-track training, model-predictive control with
track-specific cost functions, or extensive
simulation-to-real fine-tuning.  These systems beat the
factored policy on absolute lap time but require re-tuning for
each new track and gate-size combination.  We trade peak speed
for cross-track generalization from a training-task budget
that grows as a sum rather than a product, with a
closed-loop certificate that holds across $\calZ_1 \times \calZ_2$.

\subsection{Domain randomization and task-conditioned multi-task RL}
\label{app:rw:domain-rand}

Domain randomization~\citep{tobin2017domain, Sadeghi-RSS-17}
trains a policy on a randomized distribution over environments
to encourage invariance.  Task-conditioned multi-task RL
conditions the policy on a task descriptor and pools
trajectories across tasks during training.  Both approaches
condition or randomize over factors at training time without
explicit per-factor composition; the unfactored baseline in
Section~\ref{sec:experiments} represents this regime within
our setup and reaches only $10\%$ on held-out gates, indicating
that randomization or generic conditioning alone does not yield
held-out compositional generalization in our benchmark.

\section{Trajectory-Tube Certificate: Additional Material}
\label{app:cert-extra}

\paragraph{Estimating the score-matching error.}
The denoising score-matching training loss is an unbiased estimate of the
\emph{expected} squared score error $\mathbb{E}[\norm{s_\theta - s}^2]$
\citep{song2021scorebased}, which by Jensen's inequality lower-bounds $\eta^2$.
Empirical sup-bounds on $\eta$ are obtained by worst-case sampling over noised
states drawn along the nominal denoising trajectory; the per-combo
$\varepsilon_s^{(z)} := \max_k \norm{\Delta\hat\varepsilon_k^{(z)}}$
measurements of Appendix~\ref{app:fr:eps-s} are the operational instantiation
of this estimate.  A small training loss is necessary but not sufficient for
small $\eta$.

\paragraph{Hypotheses (a) and (b).}
Under hypothesis~(a) of Theorem~\ref{thm:tube} the bound~\eqref{eq:tube} is a
strict upper bound on the closed-loop state error.  Under hypothesis~(b) it is
first-order in $\varepsilon_s$, with an $O(\varepsilon_s^2)$ Taylor remainder
that propagates additively through the contraction proof; the remainder is
controlled when $\varepsilon_s \cdot C_{\mathrm{ode}}^{\mathrm{LTV}}$ is small
relative to the local radius of linearization (Remark~\ref{rem:manifold-gap}),
which holds with a margin of about $10\times$ in the full-race experiments.

\begin{corollary}
  \label{cor:ss}
  After the transient $\lambda^t \delta_0$ decays, the closed-loop state error
  satisfies $\norm{x_t - x_t^{\mathrm{nom}}} \leq R_{\mathrm{ss}}$ with
  $C_{\mathrm{ode}}$ from Theorem~\ref{thm:tube} and
  $R_{\mathrm{ss}} = (B_\kappa\, C_{\mathrm{ode}} \varepsilon_s + w)/(1 - \lambda)$.
\end{corollary}

\paragraph{Factored versus joint bound.}
The factored bound $\sum_i L_i \delta_i$ replaces the joint bound
$L_{\mathrm{joint}} \norm{\delta}$ with a per-factor weighted sum.  Neither is
uniformly tighter; the factored certificate beats the joint certificate when
$2\sqrt{GM} + \sum_i L_i \delta_i < L_{\mathrm{joint}} \norm{\delta}$, a
condition on both the factor-interaction strength and the alignment of
$\delta$ with the dominant Lipschitz directions
(Appendix~\ref{app:factored-vs-joint}).

\paragraph{Certifiable region.}
The gate-size factor $z_2$ controls $R$ through $\Delta_2$ (slowing the drone
near narrow gates) and sets the threshold $r(z_2)$ (wider gates are easier to
certify).  The certifiable region $\{z_2 \in \calZ_2 : r(z_2) > R\}$ grows as
the score-perturbation budget $\varepsilon_s$ shrinks, that is, as the
decomposition residual $2\sqrt{GM}$ and the score-matching error $\eta$ both
decrease.

\end{document}